\documentclass[final]{cvpr}

\usepackage{times}
\usepackage{epsfig}
\usepackage{graphicx}
\usepackage{amsmath}
\usepackage{amssymb}

\usepackage[numbers,sort,compress]{natbib}
\usepackage[utf8]{inputenc} % allow utf-8 input
\usepackage[T1]{fontenc}    % use 8-bit T1 fonts
\usepackage{url}            % simple URL typesetting
\usepackage{booktabs}       % professional-quality tables
\usepackage{amsfonts}       % blackboard math symbols
\usepackage{nicefrac}       % compact symbols for 1/2, etc.
\usepackage{microtype}      % microtypography
\usepackage{wrapfig}

\usepackage{color}

\usepackage{amsmath,amssymb,amsthm}
\usepackage{amsfonts}
\usepackage{subcaption}

\theoremstyle{plain}
\newtheorem{theorem}{Theorem}[section]

\newcommand{\mr}{\mathbb{R}}
\newcommand{\expec}{\mathbb{E}}
\renewcommand{\vec}[1]{\mathbf{#1}}

\usepackage[pagebackref=true,breaklinks=true,colorlinks,bookmarks=false]{hyperref}

\usepackage{cleveref}

\def\cvprPaperID{11434} % *** Enter the CVPR Paper ID here
\def\confYear{CVPR 2021}

\begin{document}

%%%%%%%%% TITLE
\title{Limitations of Post-Hoc Feature Alignment for Robustness}

\author{\textbf{Collin Burns}\\
UC Berkeley\\
% {\tt\small collinburns@berkeley.edu}
\and
\textbf{Jacob Steinhardt}\\
UC Berkeley\\
}

\maketitle

%%%%%%%%% ABSTRACT
\begin{abstract}
    Feature alignment is an approach to improving robustness to distribution shift that matches the distribution of feature activations between the training distribution and test distribution.
    A particularly simple but effective approach to feature alignment involves aligning the batch normalization statistics between the two distributions in a trained neural network.
    This technique has received renewed interest lately because of its impressive performance on robustness benchmarks.
    However, when and why this method works is not well understood. 
    We investigate the approach in more detail and identify several limitations.
    We show that it only significantly helps with a narrow set of distribution shifts and we identify several settings in which it even degrades performance.
    We also explain why these limitations arise by pinpointing why this approach can be so effective in the first place.
    Our findings call into question the utility of this approach and Unsupervised Domain Adaptation more broadly for improving robustness in practice.
\end{abstract}

%%%%%%%%% BODY TEXT
\section{Introduction}\label{sec:into}

A foundational assumption made in most of machine learning is that the training distribution is identical to the test distribution. However, this assumption is commonly violated in practice, which can substantially decrease the performance of models \citep{hendrycks2019benchmarking,recht2019imagenet}.
This can be especially problematic in high-stakes applications such as autonomous vehicles.
% \textbf{citep} 
One way of improving robustness is to exploit unlabeled test data to adapt the model to the new distribution. This %setting, in which one has labeled data from a source distribution and unlabeled data from a different but related target distribution, 
process is called Unsupervised Domain Adaptation (UDA) \citep{deep_da_survey}.

A common approach in UDA, known as feature alignment or domain alignment, is to align the feature activations between the source and target distributions %in some feature space 
\citep{coral, uda_backprop, domain_adv_training, deep_da_survey, autodial, deep_activation_matching, shu2018dirt, conditional_adversarial, coreg_alignment_uda, feature_whitening_uda, deng2019cluster}. 
Feature alignment has also been applied beyond UDA in domains such as causal inference \citep{johansson2016learning,shalit2017estimating}. 
Simple forms of feature alignment normalize the features of a trained model so that the training set and test set have the same first and second order statistics in some feature space \citep{coral,adabn}, while other approaches match distributions in more complicated ways, such as by being indistinguishable to an adversarial discriminator \citep{conditional_adversarial, domain_adv_training}. %We focus on the first type.

We focus on one simple feature alignment method: Adaptive Batch Normalization (AdaBN) \citep{adabn}. Like many other popular and effective feature alignment methods (e.g. \citet{coral,deep_coral,autodial,feature_whitening_uda,wang2019transferable}), AdaBN is normalization-based, meaning it matches first and second order statistics between the two feature distributions. It is also a post-hoc method, meaning it aligns features for a model that has already been trained, making it particularly simple and applicable even for unforseen distribution shifts. 
Given a neural network trained on source data with Batch Normalization (BN) \citep{ioffe2015batch}, AdaBN re-estimates the BN statistics of that model using the target data. 
In other words, AdaBN aligns the mean and variance of each channel in the network across the two distributions. 

Despite its simplicity, in recent work \citet{Schneider2020ImprovingRA, Nado2020EvaluatingPB} showed that aligning batch norm statistics between the train and test distributions can be used to achieve state-of-the-art accuracy on the robustness benchmark ImageNet-C \citep{hendrycks2019benchmarking}.
\citet{Schneider2020ImprovingRA} argues that we should therefore start using normalization-based feature alignment methods whenever we evaluate robustness. 
\citet{Nado2020EvaluatingPB} additionally finds that aligning BN statistics does not help as much for some other types of distribution shift.
However, neither paper describes why this method works well on ImageNet-C or why it does not help as much with other types of distribution shift.
% , which they say ``indicates that the method is worthy of additional study.'' 

We build on this work by investigating when and why methods like AdaBN help. Our findings include:
\begin{itemize}
    \item Showing that aligning BN statistics can actually \emph{degrade} accuracy on several types of distribution shift, both conceptually and in practice.
    \item Identifying implicit symmetry assumptions made by these methods and showing how violations of these assumptions can cause performance degradation.
    \item Demonstrating and explaining how aligning BN statistics primarily helps with distribution shifts that involve changes in local image statistics.
\end{itemize}

Our findings have several implications. 
While aligning BN statistics is an effective method for improving robustness in some settings, it only significantly helps on a narrow set of distribution shifts and can even degrade performance. 
These limitations may prevent it from being useful in practical applications. 
Furthermore, we find that existing justifications of feature alignment are inadequate for explaining when and why these methods work. Future work on UDA should explicitly identify the properties of data distributions and neural networks that these methods rely on in practice.
Finally, some of our findings apply to UDA more broadly, calling into question whether UDA is a strong approach to improving the robustness of machine learning systems in the first place. 
More work is therefore needed to make UDA practical for improving robustness.

\begin{figure*}[t]
    \vspace{-10pt}
    \centering
    \begin{subfigure}[t]{0.33\linewidth}
        \centering
        \includegraphics[width=60mm]{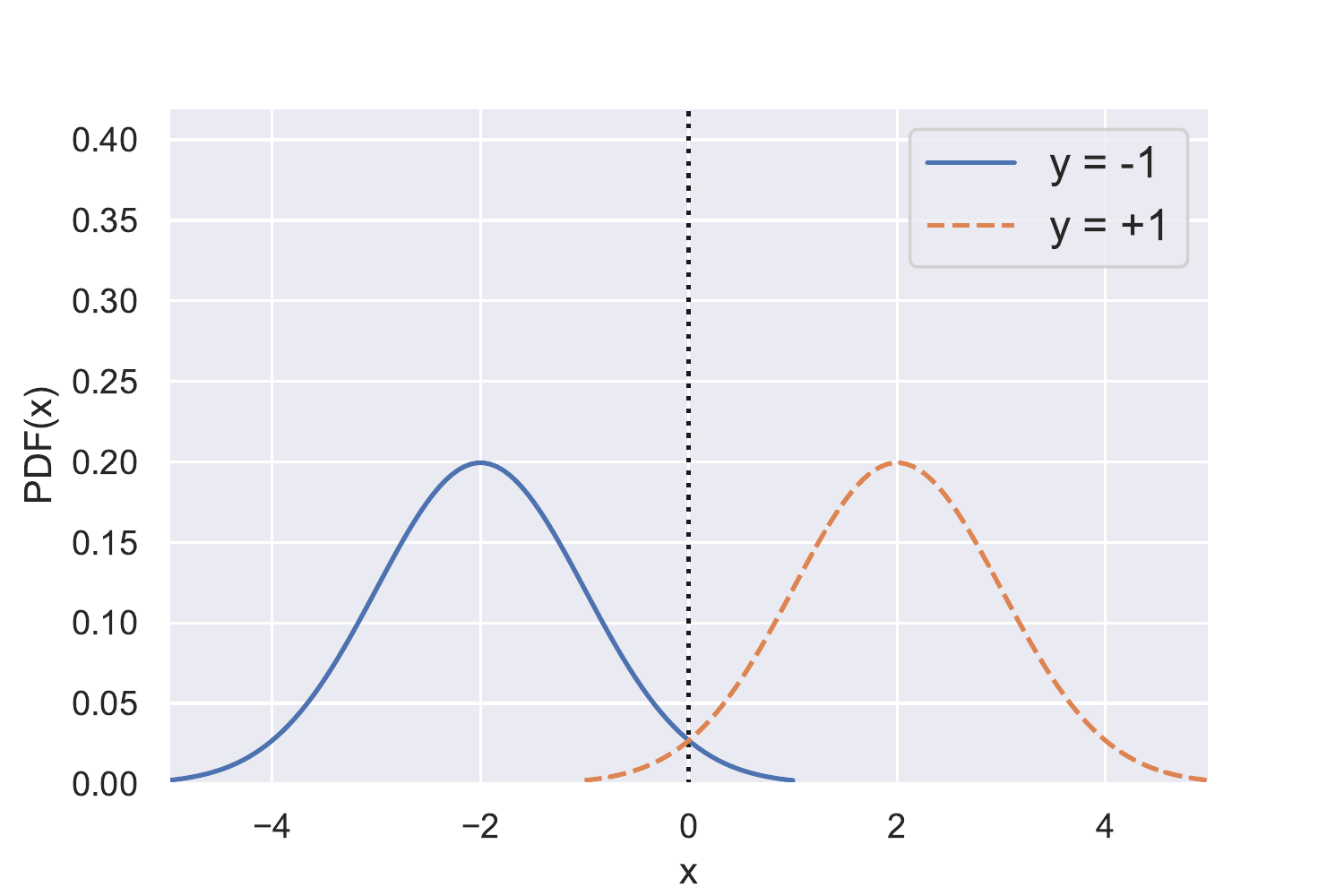}
        \caption{Source distribution.}
        \label{subfig:label_shift_source}
    \end{subfigure}
    \begin{subfigure}[t]{0.33\linewidth}
        \centering
        \includegraphics[width=60mm]{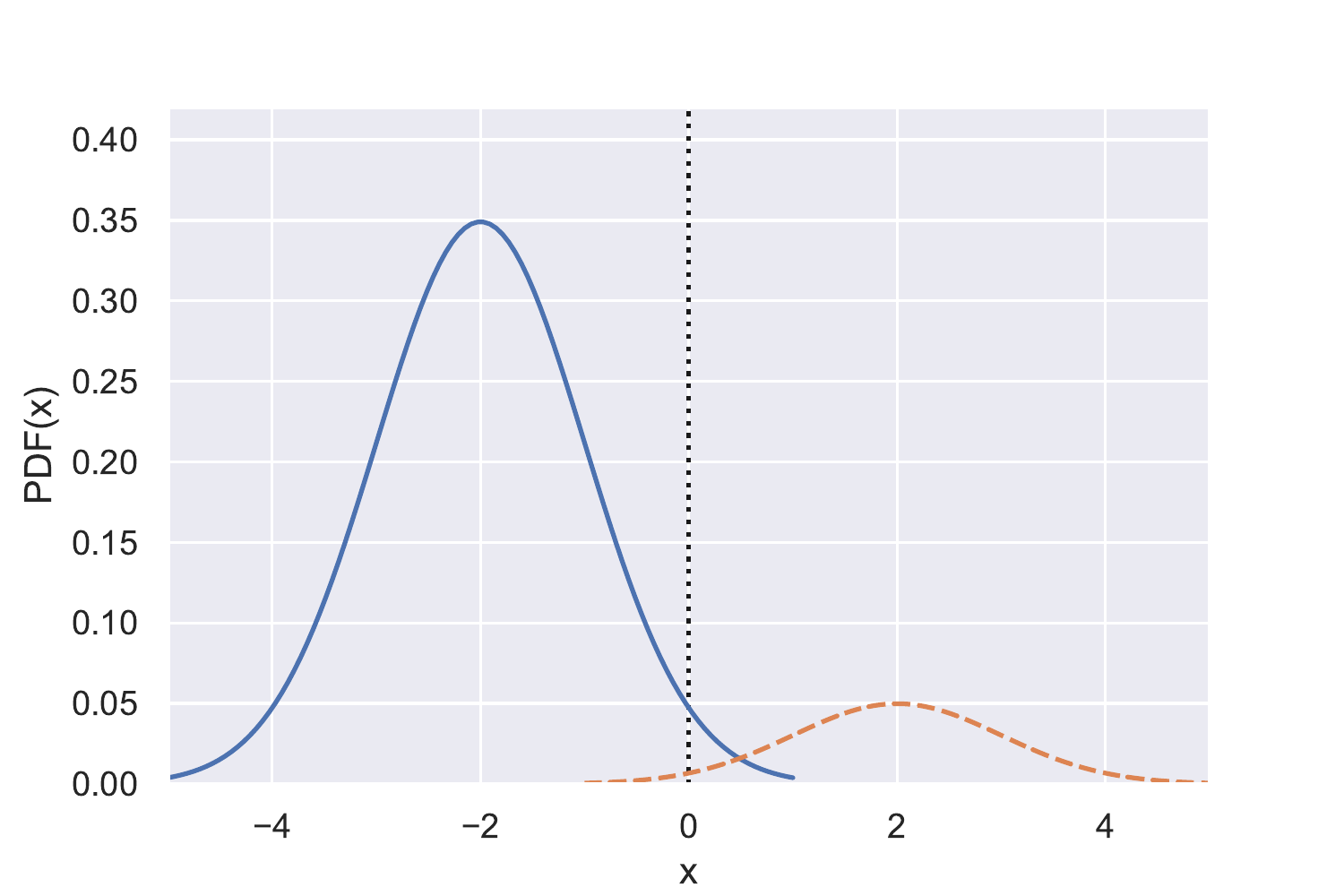}
        \caption{Target distribution.}
        \label{subfig:label_shift_target}
    \end{subfigure}
    \begin{subfigure}[t]{0.33\linewidth}
        \centering
        \includegraphics[width=60mm]{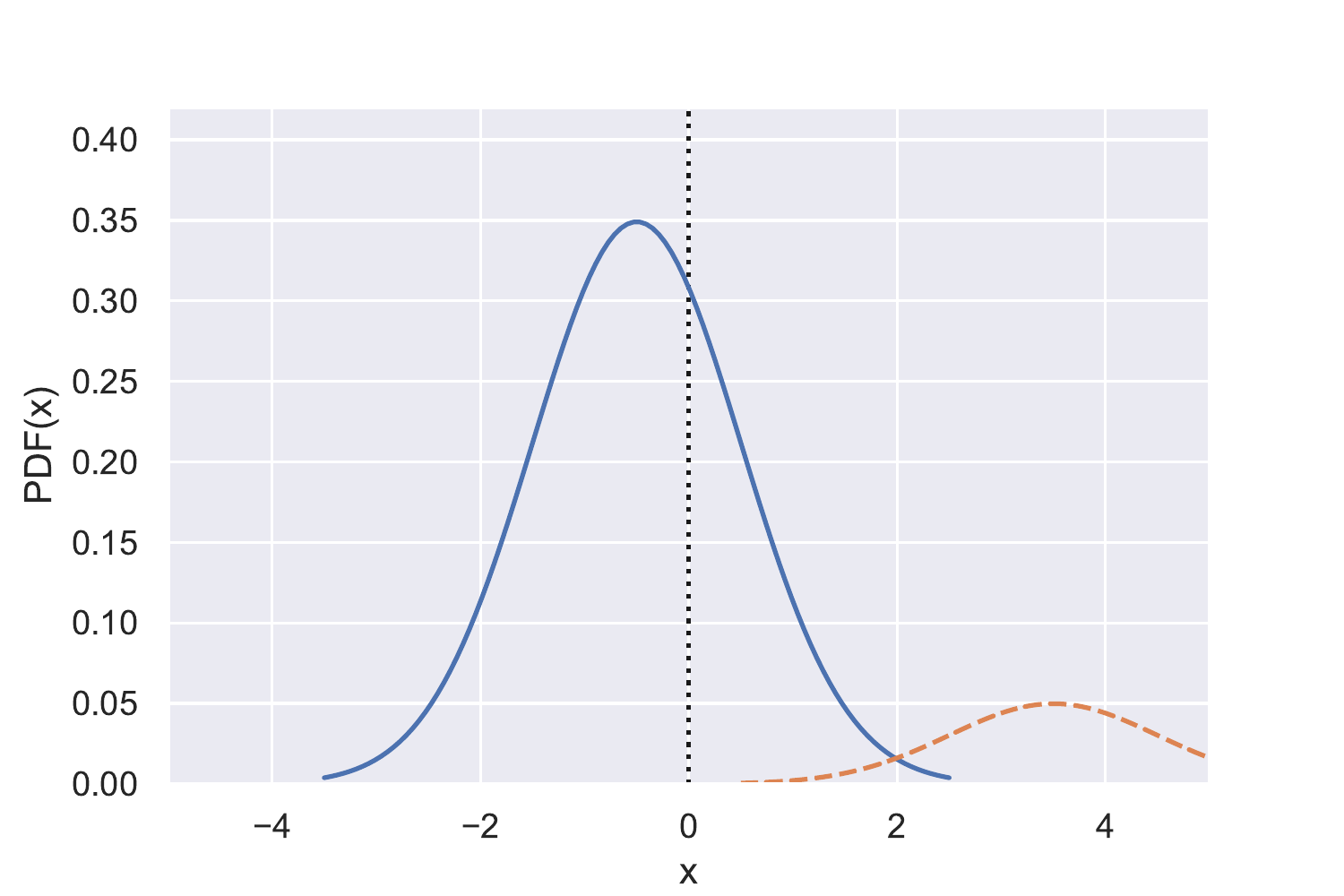}
        \caption{Mean-normalized target distribution.}
        \label{subfig:label_shift_normalized_target}
    \end{subfigure}
    \caption{
    An illustration of how aligning the means between the original and shifted distributions can hurt accuracy when there is label shift. Aligning the variances does not change the accuracy in this case. The blue curve is the PDF for $y=-1$ and the orange (dashed) curve is the PDF for $y=+1$. The dashed line indicates the decision boundary of the classifier. Both classes are initially equally likely (\Cref{subfig:label_shift_source}). After shifting the label distribution, the accuracy of the original classifier remains high (\Cref{subfig:label_shift_target}), but decreases after normalizing the mean (\Cref{subfig:label_shift_normalized_target}).}
    \label{fig:label_shift_failure}%\vspace{-10pt}
\end{figure*}

\begin{table*}[t]
\caption{Accuracy of each method on CIFAR-10 (C-10), TinyImageNet (TIN), ImageNet (IN), CIFAR-10-C (C-10-C), TinyImageNet-C (TIN-C), ImageNet-C (IN-C), ImageNetV2 (INV2), and Stylized ImageNet (SIN).} %AdaBN substantially helps accuracy on the corruption datasets and Stylized ImageNet, while AdaBN + Aug does even better in several cases.
\label{tbl:adabn_accuracies}
% \vskip 0.15in
\begin{center}
\begin{small}
\begin{sc}
\begin{tabular}{lccc|ccc|cc}
\toprule
Method & C-10 & TIN & IN & C-10-C & TIN-C & IN-C & INV2 & SIN\\
\midrule
Original model        & 94.8 & 63.8 & 76.1 & 72.3 & 24.7 & 38.1 & 63.2 & 7.1\\
AdaBN                 & 92.8 & 60.3 & 75.6 & 83.6 & 40.1 & 46.9 & 60.9 & 10.2\\
% AdaBN + Aug          & 94.8 & 64.0 & 76.0 & 86.7 & 41.8 & 43.3 & 63.8 & 8.4\\
\bottomrule
\end{tabular}
\end{sc}
\end{small}
\end{center}
\vskip -0.1in
\end{table*}

\section{Related Work}\label{sec:related_work}
We focus on feature alignment methods that work by aligning the Batch Normalization statistics between the source and target distributions for a trained neural network. In this section, we describe how this relates to other feature alignment methods, and we describe why existing justifications for feature alignment do not adequately explain their practical success.

\paragraph{Feature Alignment Methods.}
Several UDA methods closely resemble AdaBN by similarly aligning normalization statistics of trained models. \citet{coral} whiten and re-color the target distribution to match the mean and covariance of the source distribution in the input. \citet{deep_coral} extend this by matching the mean and covariance in a neural network layer, rather than in the input. %, making it almost a special case of AdaBN. 
Because these are post-hoc methods based on normalization like AdaBN, our findings directly apply to them as well.

Some UDA methods are normalization-based but require modifying the \textit{training} of neural networks as well.
\citet{autodial} modify AdaBN by learning a linear combination of source and target Batch Normalization statistics. 
\citet{wang2019transferable} introduce a new layer for UDA that uses domain-specific Batch Normalization statistics and that automatically adapts to the transferability of different channels.
Some, but not all, of our findings apply to these methods as well.

In a related vein, adversarial alignment methods such as \citet{domain_adv_training,conditional_adversarial} learn feature representations for which a discriminator cannot distinguish source and target data. 
Unlike the normalization-based approaches that we focus on in this work, adversarial methods aim to learn feature representations that are completely indistinguishable instead of only matching first and second order statistics, and again modify the training of networks, which can be expensive.
These methods can improve performance, but they are also much less efficient than post-hoc feature alignment methods.

\paragraph{Justifications of Feature Alignment are Inadequate.}
Many papers that introduce feature alignment methods intuitively suggest that matching feature distributions makes the features more domain-invariant and consequently mitigates the effects of distribution shift \citep{coral,deep_coral,uda_backprop,autodial}. 
However, aligning the features between two distributions is not sufficient for good test performance in general because aligning the marginal distributions $p_S(x)$ and $p_T(x)$ in some feature space may not align the \emph{class-conditional} distributions $p_S(x | y)$ and $p_T(x | y)$ \citep{on_invariant_reps, johansson2019support}. 

Some papers (e.g. \citet{conditional_adversarial,domain_adv_training}) motivate aligning feature distributions by referring to \citet{ben2010theory}, which introduces generalization bounds for UDA. For a given hypothesis class $\mathcal{H}$ and feature space, these bounds guarantee good test performance as long as (i) the two distributions are ``indistinguishable'' with respect to $\mathcal{H}$, and (ii) there is a hypothesis $h \in \mathcal{H}$ that simultaneously does well on both distributions. However, \citet{on_invariant_reps,johansson2019support} recently described problems with this theory, and in \Cref{appx:subsec:uninformative} we argue that these generalization bounds are probably vacuous in practice. 
In contrast to this work, we focus on empirically understanding when and why aligning BN statistics works \emph{in practice}. 

Several impossibility theorems show that successful UDA requires strong assumptions on the source and target distributions \citep{ben2010impossibility,ben2012hardness}. 
Nevertheless, many feature alignment methods are effective in practice.
This raises the question: \emph{What properties of distribution shifts and neural networks does feature alignment exploit to improve robustness?} 
We answer this question for AdaBN in the process of investigating its limitations.

\section{Failure Modes of AdaBN}\label{sec:failure}
In this section, we characterize when normalization-based methods \emph{hurt} accuracy. Prior work showed that feature alignment can degrade performance under label shift, i.e. $p_S(y) \neq p_T(y)$ \citep{on_invariant_reps, johansson2019support,redko2019}. We extend these earlier observations by showing that label shift also has a more severe impact on \textit{deep} layers than on \textit{shallow} layers. 

We then construct two additional failure modes that can occur even when the label distribution doesn't change, i.e. $p_S(y) = p_T(y)$, and even under the covariate shift assumption, i.e. when $p_S(y | x) = p_T(y | x)$. In particular, we show that normalization-based alignment methods can fail when either different \emph{examples} or \emph{spatial locations} are shifted in qualitatively different ways, and we show that both types of shift can arise in practice.
This suggests that these methods would be unreliable in safety-critical applications involving unforeseen distribution shifts.

For each of the three failure modes we exhibit, we first provide a simple conceptual example of why the failure mode is possible, then demonstrate the failure on real data.

\subsection{Experimental Setup}\label{subsec:setup}
We begin by describing the experimental setup that we use for the remainder of the paper.
Code for the experiments is available at \url{https://github.com/collin-burns/feature-alignment}.

\begin{figure*}[t]
    \vspace{-10pt}
    \centering
    \begin{subfigure}[t]{0.49\linewidth}
        \centering
        \includegraphics[scale=0.40]{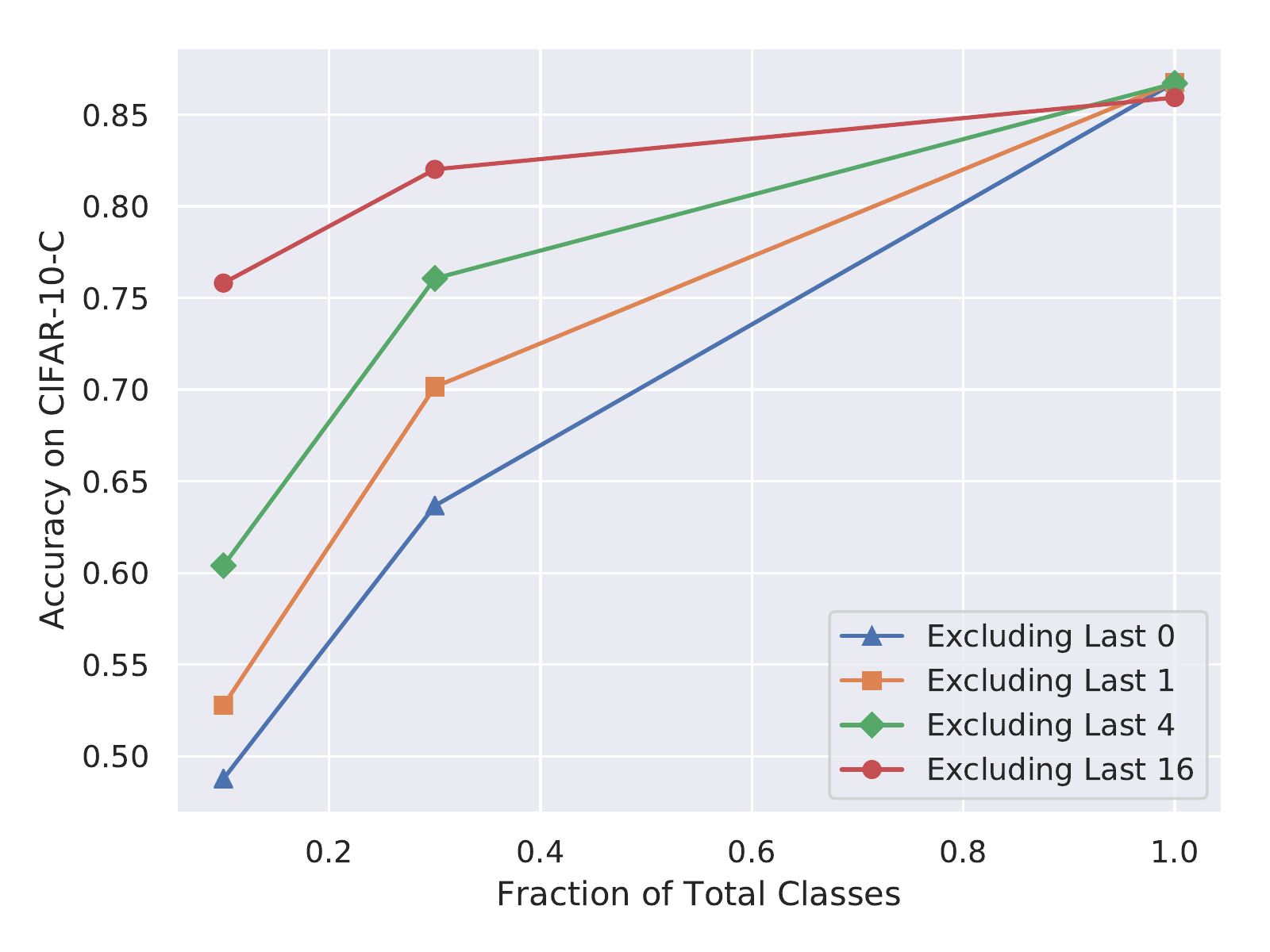}
        \caption{CIFAR-10-C Excluding Last Layers}
        \label{subfig:excl_last}
    \end{subfigure}
    \begin{subfigure}[t]{0.49\linewidth}
        \centering
        \includegraphics[scale=0.40]{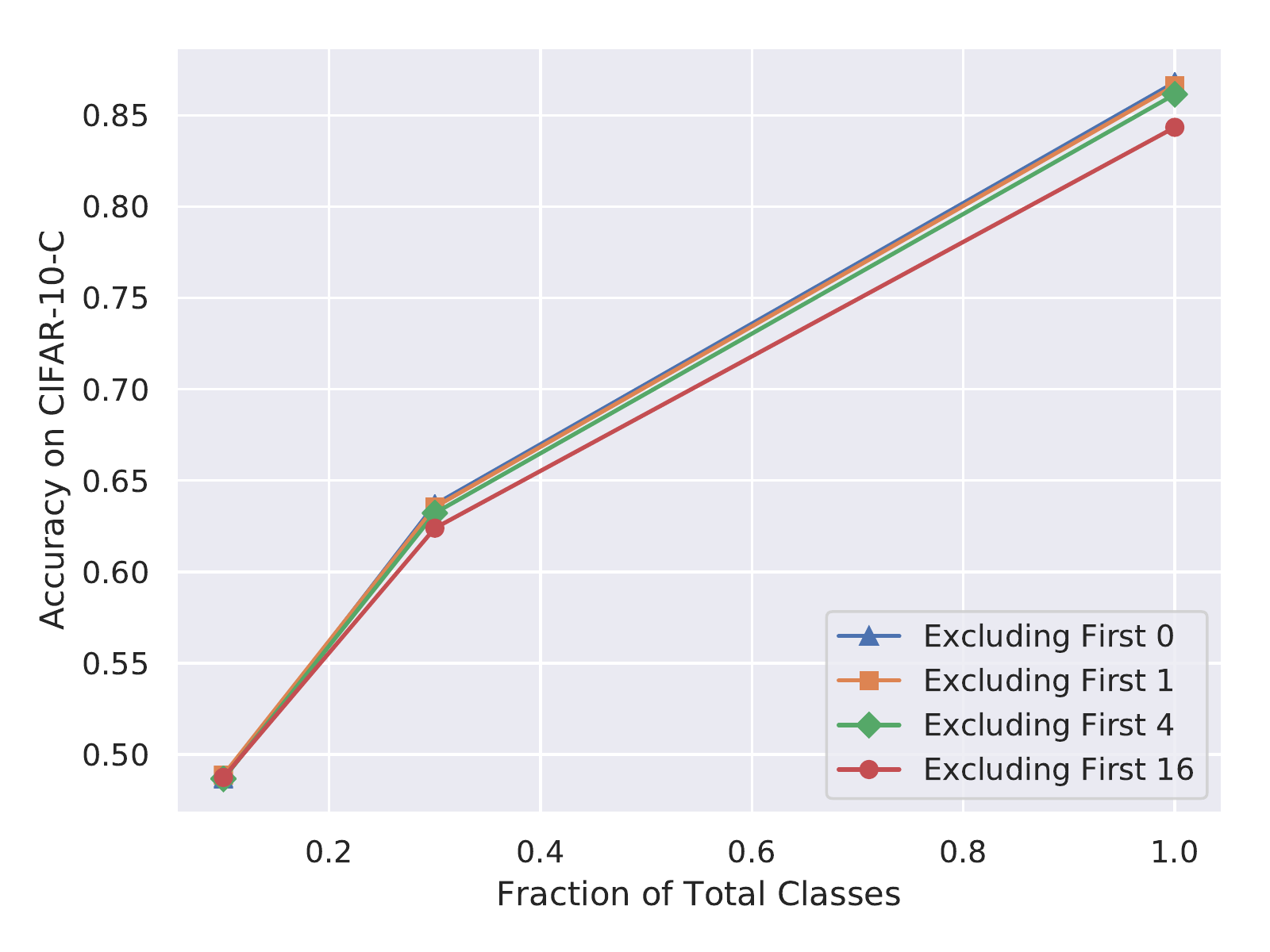}
        \caption{CIFAR-10-C Excluding First Layers}
        \label{subfig:excl_first}
    \end{subfigure}
    \caption{The effect of updating the Batch Normalization statistics (AdaBN) on CIFAR-10-C in all but the last $k$ BN layers (left) and all but the first $k$ BN layers (right) as a function of the number of classes kept in the target shift, for $k \in \{0, 1, 4, 16\}$. There are $37$ BN layers in total. AdaBN does worse as the number of classes decreases. We can mitigate this decrease by excluding some of the final Batch Normalization layers from feature alignment, but not by excluding some of the first Batch Normalization layers. This indicates that deep layers are more sensitive to label shift than shallow layers.}
    \label{fig:label_shift}
\end{figure*}

\paragraph{Datasets.}
We evaluate models on a diverse set of distribution shift datasets. In several experiments we use the robustness benchmarks CIFAR-10-C, TinyImageNet-C, and ImageNet-C \citep{hendrycks2019benchmarking}. These datasets include $15$ noise, blur, weather, and digital corruptions with $5$ severities for each; we apply AdaBN on each corruption and severity independently then average the resulting accuracies. We also run experiments on ImageNetV2 \citep{recht2019imagenet} and Stylized ImageNet \citep{stylized_imagenet}. ImageNetV2 was constructed by trying to reproduce how the ImageNet dataset was collected. This distribution shift reduces the accuracy of our ImageNet-trained model from $76\%$ to $63\%$. Stylized ImageNet \citep{stylized_imagenet} changes the texture and style of ImageNet images in a variety of ways. This more severe shift reduces the accuracy all the way down to $10\%$. As we are working in the context of domain adaptation, for each dataset we assume we have access to all of the unlabeled target data. 

Using these datasets has two main advantages. First, ImageNet-C, ImageNetV2, and Stylized ImageNet are qualitatively distinct shifts, but they are readily comparable because they are all based on ImageNet. Second, the corruption datasets make it possible to compare different severities of shift while controlling for other factors, which we make use of in some experiments. 

\paragraph{Models.}
We use the pre-trained ResNet-$50$ model included in the \texttt{torchvision} package \citep{pytorch} for all ImageNet experiments. For all CIFAR-10 and TinyImageNet experiments, we use a $40$-$2$ WideResNet \citep{wrn} trained for $100$ epochs using SGD with momentum $0.9$, initial learning rate $0.1$, weight decay $0.0005$, dropout rate $0.3$, and batch size $128$. For data augmentation, we use random cropping with zero-padding $4$ and random horizontal flips. We show the performance of these models on each dataset in \Cref{tbl:adabn_accuracies}.

\subsection{Shifted Label Distribution}\label{subsec:shifted_label_dist}
We now show how it can hurt to normalize the feature distributions when $p_S(y) \neq p_T(y)$.
Similar observations were made in prior work \citep{on_invariant_reps, johansson2019support,redko2019}, but we extend this by investigating how it occurs in more detail and by also showing that deeper layers are more sensitive to label shift than shallow layers.

\paragraph{Conceptual Example.}
Consider binary classification when $y$ is sampled uniformly from $\{-1, 1\}$, and $x~\sim~\mathcal{N}(2y, 1)$. The Bayes classifier is $f(x) = \text{sign}(x)$. Suppose we shift the label distribution so that $p(y = -1) = \frac{7}{8}$ for the target distribution. The new mean is then $-\frac{3}{2}$. If we normalize the mean to match the original mean of zero, this pushes the $y = -1$ mode from $\mathcal{N}(-2, 1)$ to $\mathcal{N}(-\frac{1}{2}, 1)$ and increases the classification error substantially, as illustrated in \Cref{fig:label_shift_failure}. This is despite the fact that the classifier would have had high accuracy without any normalization. 

\paragraph{In Practice.} We now exhibit this issue on CIFAR-10. From the discussion above, we should expect the accuracy of AdaBN to degrade as it is applied to a smaller fraction of classes. We confirm this and show the results in \Cref{fig:label_shift} (blue curve). Specifically, we evaluate the accuracy of AdaBN applied to subsets of CIFAR-10-C classes, while still allowing the classifier to output any of the $10$ classes. Making some classes occur with probability zero is an extreme form of class reweighting, but one that could still arise in practice.
For simplicity, we use the first $k$ classes for different values of $k$. 
In the worst case of a single class, the accuracy falls below $50\%$. 

\paragraph{Shallow vs Deep Layers.} Intuitively, shallow layers capture low-level information like edges and colors, which should be mostly class-agnostic, while deeper layers capture more abstract, class-specific representations. This suggests that only updating the Batch Normalization statistics in the earlier layers may mitigate the drop in accuracy caused by applying AdaBN under label shift. 

We confirm this and show the results in \Cref{subfig:excl_last}. When one doesn't update the last $16$ (out of $37$) Batch Normalization layers, accuracy remains high even when AdaBN is applied to a single class. To check that this is due to excluding the \emph{final} layers, we also test not updating the {first} $k$ Batch Normalization layers and confirm that it does not improve performance (\Cref{subfig:excl_first}). We find similar results on other datasets; see \Cref{appx:subsec:subsets} for details.

\subsection{Shifted Spatial Locations}\label{subsec:diff_feature_shift}

A second type of failure mode occurs when there are different shifts for different spatial locations. This can occur if, for example, a border is added to every image, as this results in the distribution of boundary pixels changing dramatically without the distribution of interior pixels changing at all. We illustrate this in \Cref{fig:black_border_visualization} (top row). Unlike the previous failure mode, this can arise even under the covariate shift assumption and when the class distribution is fixed. 

\paragraph{Conceptual Example.} Again consider binary classification when $y$ is sampled uniformly from $\{-1, +1\}$. Let $\vec{x} = (x_1, x_2)$, where $x_1 \sim \mathcal{N}(4+2y, 1)$ and $x_2 \sim \mathcal{N}(4, 1)$. The classifier $f(\vec{x}) = \text{sign}(x_1-4)$ has high accuracy.

We use the features $x_1$ and $x_2$ to model different spatial dimensions in a convolutional channel. Since Batch Normalization computes the mean and variance over both a batch of examples and all spatial locations within a channel, we simultaneously normalize $\vec{x}$ over both samples and coordinates. Specifically, AdaBN matches the mean and variance of a shifted input $\vec{x}$ by transforming the input to be
\begin{equation}\label{eqn:adabn}
    \tilde{\vec{x}} = \frac{\sigma_s}{\sigma_t}(\vec{x} - \mu_t\cdot \vec{1}) + \mu_s\cdot \vec{1} \,,
\end{equation}
where $\vec{1} := (1, 1) \in \mr^2$. Imagine we shift the distribution by making $x_2 = 0$ for each example. This doesn't change the accuracy of $f$, but it decreases the mean and variance of $\vec{x}$. The original mean and variance were the average mean and variance over each dimension: $\mu_s = \frac{1}{2}(\expec[x_1]+\expec[x_2]) = 4$ and $\sigma^2_s = \frac{1}{2}(\text{var}(x_1)+\text{var}(x_2)) = 1$. Under the shift, $\expec[x_2] = \text{var}(x_2) = 0$, so the mean and variance become $\mu_t = \frac{1}{2}(\expec[x_1]+0) = 2$ and $\sigma^2_t = \frac{1}{2}(\text{var}(x_1)+0) = \frac{1}{2}$.

For the values given above, we have $\tilde{x}_1 = \sqrt{2}(x_1 - 2) + 4$. The mode corresponding to $y=-1$ is initially centered at $x_1 = 2$, so after normalizing it shifts to $\tilde{x}_1 = 4$. Since the decision boundary of $f$ passes through $x_1 = 4$, the new error of $f$ conditioned on $y=-1$ is $\frac{1}{2}$, which is higher than it was before applying AdaBN. 

% \begin{wraptable}{R}{0.36\textwidth}
\begin{table}[h]
\caption{Accuracy on the Black Border distribution shift.}
\label{tbl:black_border}
% \tiny
% \vskip 0.15in
% \vspace{-15pt}
\vspace{-5pt}
\begin{center}
\begin{small}
\begin{sc}
\begin{tabular}{lcc}
\toprule
Method & C-10 & TIN \\
\midrule
Original model           & 65.0 & 22.6\\
AdaBN                    & 52.5 & 11.8 \\
\bottomrule
\end{tabular}
\end{sc}
\end{small}
\end{center}
\vspace{-10pt}
\end{table}
% \end{wraptable}

\paragraph{In Practice.} We now exhibit an analogous failure mode on real data. Our example uses the ``black border'' transformation, where we remove all boundary pixels by replacing them with zero. Similar to the conceptual example, this shifts the distribution of some spatial locations but not others. We evaluate the robustness of models to this transformation on CIFAR-10 and TinyImageNet, where we chose the width of the border to be $1/4$ the length of the image, so that $25\%$ of the area of the image remains. 

The results are given in \Cref{tbl:black_border}. Applying AdaBN to this transformation hurts accuracy relative to the original model, almost cutting it in half for TinyImageNet.

To verify that the drop in performance comes from shifted spatial locations, we visualize the effect of Batch Normalization on the activations of the model. 
We show representative channel activations after the first and twenty-first Batch Normalization layers in \Cref{fig:black_border_visualization}. 
For both layers, AdaBN changes the scale of the activations so that the mean is closer to that of a typical in-distribution activation. Since the border pixels are either darker than a normal input (top) or brighter than a normal input (bottom), matching the mean throws off the scale of the center of the image, which contains the actual content.  

\begin{figure}[t]
% \begin{wrapfigure}{l}{0.58\textwidth}
    \centering
    \begin{subfigure}[t]{0.15\textwidth}
        \centering
        \includegraphics[width=20mm]{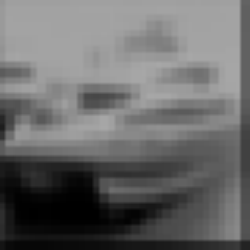}
         \captionsetup{justification=centering}
        \caption{Default model, \\original, L$1$.}
        \label{subfig:def_orig_1}
    \end{subfigure}
    \begin{subfigure}[t]{0.15\textwidth}
        \centering
        \includegraphics[width=20mm]{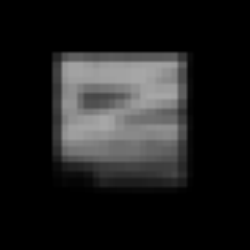}
         \captionsetup{justification=centering}
        \caption{Default model, \\transformed, L$1$.}
        \label{subfig:def_transf_1}
    \end{subfigure}
    \begin{subfigure}[t]{0.15\textwidth}
        \centering
        \includegraphics[width=20mm]{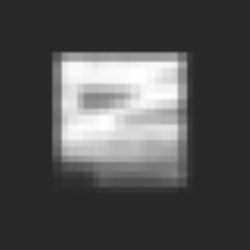}
         \captionsetup{justification=centering}
        \caption{AdaBN, \\transformed, L$1$.}
        \label{subfig:updated_transf_1}
    \end{subfigure}
    
    \begin{subfigure}[t]{0.15\textwidth}
        \centering
        \includegraphics[width=20mm]{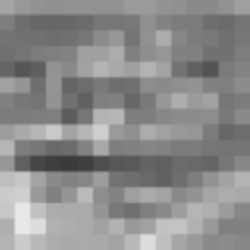}
         \captionsetup{justification=centering}
        \caption{Default model, \\original, L$21$.}
        \label{subfig:def_orig_21}
    \end{subfigure}
    \begin{subfigure}[t]{0.15\textwidth}
        \centering
        \includegraphics[width=20mm]{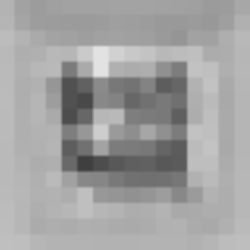}
         \captionsetup{justification=centering}
        \caption{Default model, \\transformed, L$21$.}
        \label{subfig:def_transf_21}
    \end{subfigure}
    \begin{subfigure}[t]{0.15\textwidth}
        \centering
        \includegraphics[width=20mm]{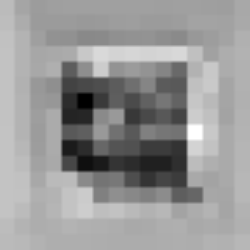}
         \captionsetup{justification=centering}
        \caption{AdaBN, \\transformed, L$21$.}
        \label{subfig:updated_transf_21}
    \end{subfigure}
    \caption{A representative channel in the $1$st Batch Normalization layer (top row, L1) and $21$st Batch Normalization layer (bottom row, L21). The scales of the images (the minimum and maximum values, corresponding to black and white) are the same. Updating the Batch Normalization statistics (\Cref{subfig:updated_transf_1,subfig:updated_transf_21}) changes the magnitude of the activations even though the original activations are more appropriate (\Cref{subfig:def_orig_1,subfig:def_orig_21}).}
    \label{fig:black_border_visualization}%\vspace{-10pt}
\end{figure}
% \end{wrapfigure}

\subsection{Shifted Examples}\label{subsec:multimodal}
Finally, we show that normalization-based feature alignment can fail if different examples are subject to different shifts. 
One example of this is with label shift, but we show that it is a more general phenomenon that can occur even when $p_S(y) = p_T(y)$. Specifically, it can naturally occur when the distribution of a single spatial location is multimodal, such as for a mixture distribution. 

\begin{figure*}[t]
    \centering
    \begin{subfigure}[t]{0.33\textwidth}
        \centering
        \includegraphics[width=60mm]{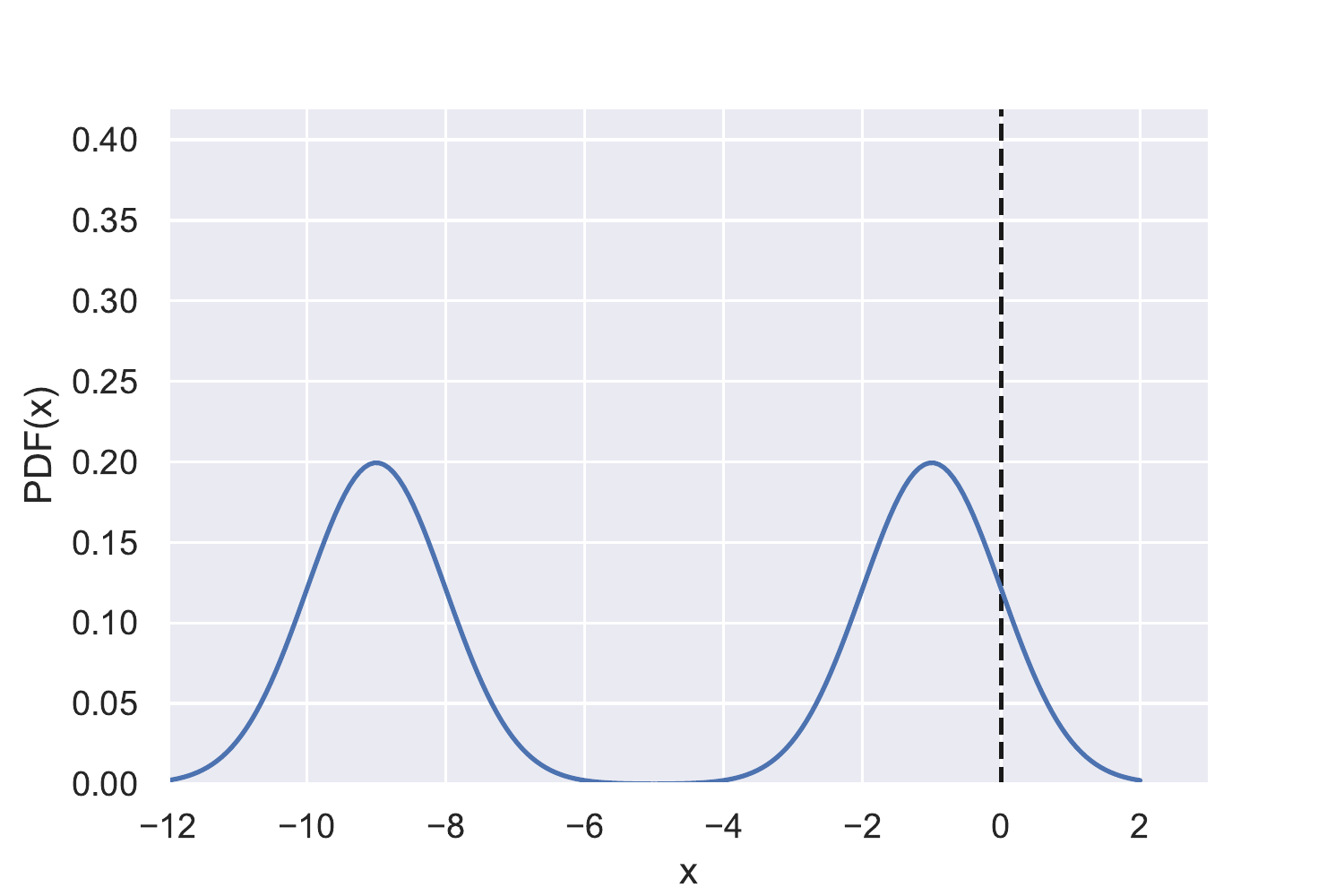}
         \captionsetup{justification=centering}
        \caption{Source distribution.}
    \end{subfigure}
    \begin{subfigure}[t]{0.33\textwidth}
        \centering
        \includegraphics[width=60mm]{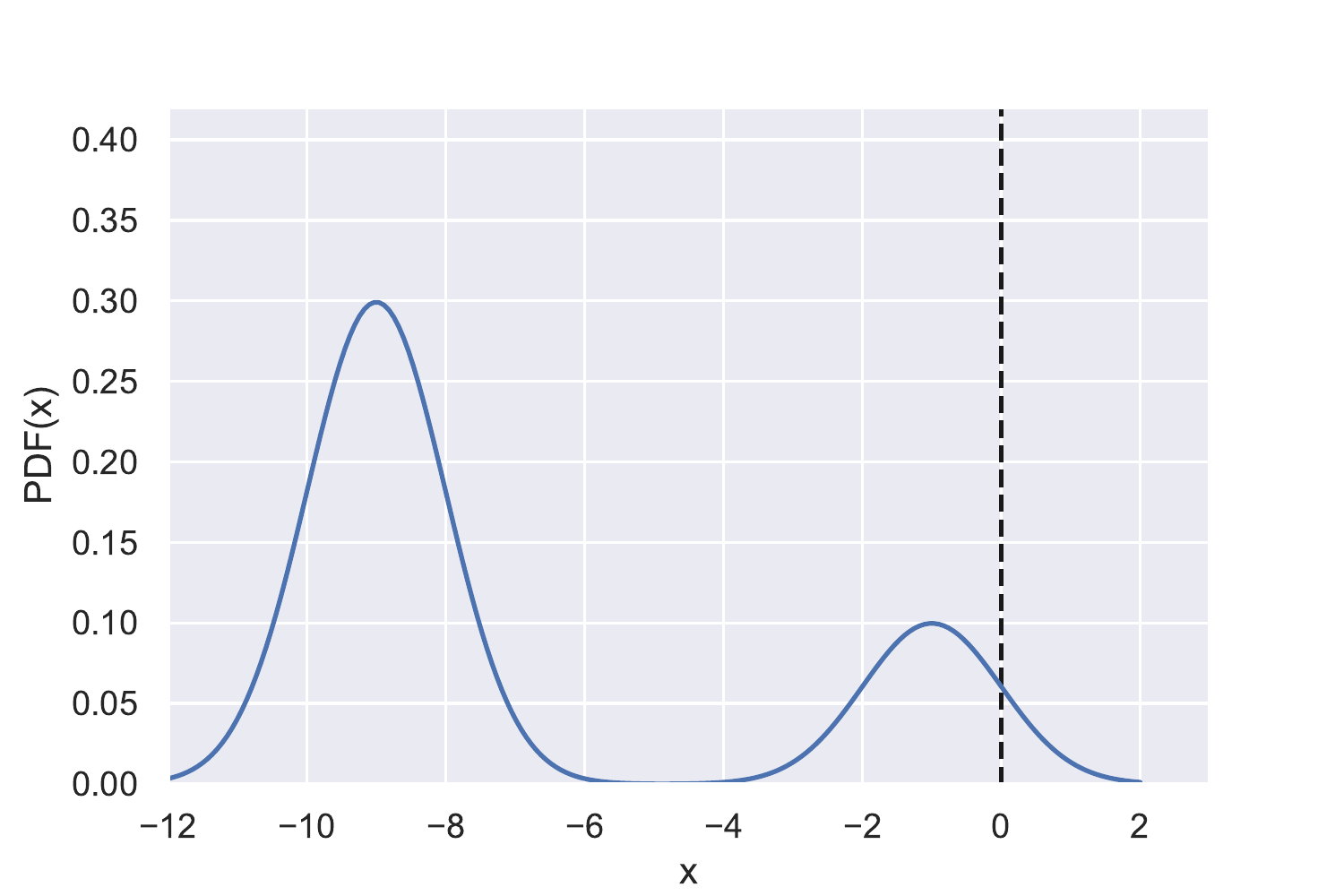}
         \captionsetup{justification=centering}
        \caption{Target distribution.}
    \end{subfigure}
    \begin{subfigure}[t]{0.33\textwidth}
        \centering
        \includegraphics[width=60mm]{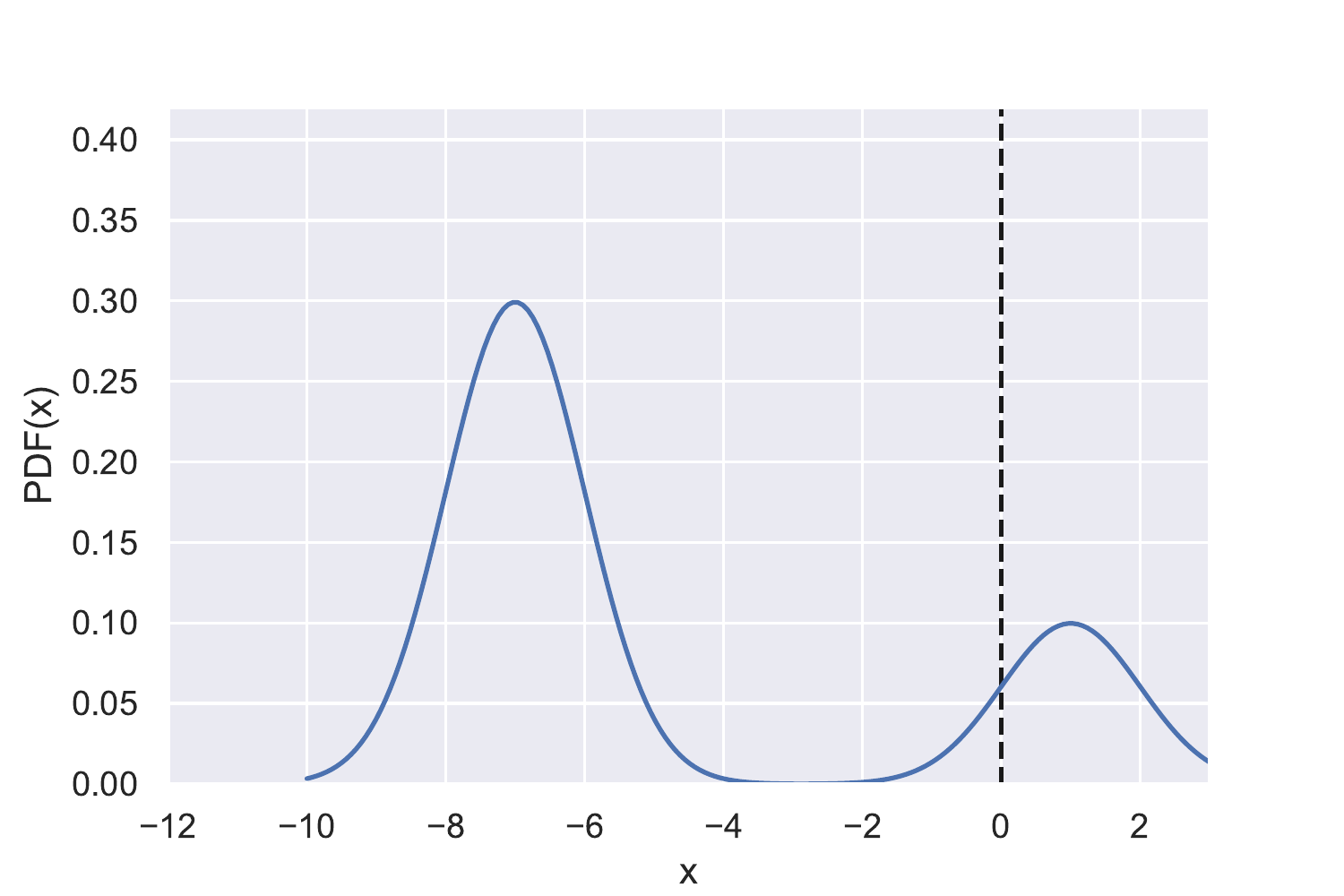}
         \captionsetup{justification=centering}
        \caption{Mean-aligned target distribution.}
    \end{subfigure}
    \caption{An illustration of how aligning the class-specific means can hurt accuracy when the class feature distribution is multi-modal. The blue curve is the PDF of $x | y=-1$ and has two modes. For clarity, we do not show the PDF of $x | y=+1$. The dashed line indicates the decision boundary of the classifier. After shifting the feature distribution so that the mode centered at $-9$ becomes more likely, the accuracy remains high. However, aligning the source and target means pushes the mode centered at $-1$ to be greater than $0$, causing a drop in accuracy.}
    \label{fig:multimodal_failure}
    % \vspace{-10pt}
\end{figure*}

% \begin{wrapfigure}{R}{0.5\textwidth}
\begin{figure}[hb]
    \centering
    \begin{subfigure}[t]{0.49\linewidth}
        \centering
        \includegraphics[width=22mm]{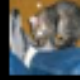}
        \subcaption{Augmented (source)}
        \label{subfig:augmented_source}
    \end{subfigure}
    \begin{subfigure}[t]{0.49\linewidth}
        \centering
        \includegraphics[width=22mm]{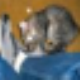}
        \subcaption{Original (target)}
        \label{subfig:original_target}
    \end{subfigure}
    \caption{An example of the augmentation used to train our CIFAR-10 and TinyImageNet models. The source (training) distribution includes images like \Cref{subfig:augmented_source} while the target (test) distribution only includes unaugmented examples like \Cref{subfig:original_target}, causing a harmless distribution shift.} 
    \label{fig:aug_example}
    % \vspace{-10pt}
\end{figure}
% \end{wrapfigure}

\paragraph{Conceptual Example.} Consider binary classification again, but this time suppose $x | y$ is a mixture of Gaussians.
For simplicity, we focus on $y=-1$ and assume that we normalize separately for each class. Define the source distribution by drawing $x | (y = -1)$ from the mixture distribution $\frac{1}{2} N(-9, 1) + \frac{1}{2} N(-1, 1)$. 
%and drawing $x | y = 1$ from $N(1, 1)$ with probability $1$. 
The classifier $f(x) = \text{sign}(x)$ initially has low error conditioned on $y=-1$. Suppose we now reweight the modes of $x | (y = -1)$ so that we instead sample from $\frac{3}{4}N(-9, 1) + \frac{1}{4}N(-1, 1)$. This decreases both the variance and the mean. Normalizing to have the original mean and variance then pushes the $N(-1, 1)$ mode to be greater than $0$, resulting in a larger classification error. We illustrate this in \Cref{fig:multimodal_failure}.

\begin{table*}[h]
\caption{Accuracy (in percent) of AdaBN with training augmentations (+ Aug) on each dataset. The change in accuracy from AdaBN is given in parentheses. AdaBN + Aug does better in most cases, especially on CIFAR-10 and TinyImageNet.}
\label{tbl:adabn_aug_accuracies}
% \vskip 0.15in
\begin{center}
\begin{footnotesize}
\begin{sc}
\begin{tabular}{lccc|ccc|cc}
\toprule
Method & C-10 & TIN & IN & C-10-C & TIN-C & IN-C & INV2 & SIN\\
\midrule
% Original model        & 94.8 & 63.8 & 76.1 & 72.3 & 24.7 & 38.1 & 63.2 & 7.1\\
% AdaBN                 & 92.8 & 60.3 & 75.6 & 83.6 & 40.1 & 46.9 & 60.9 & 10.2\\
AdaBN + Aug         & 94.8 (+2.0) & 64.0 (+3.7) & 76.0 (+0.4) & 86.7 (+3.1) & 41.8 (+1.7) & 43.3 (-3.6) & 63.8 (+2.9) & 8.4 (-1.8)\\
\bottomrule
\end{tabular}
\end{sc}
\end{footnotesize}
\end{center}
% \vskip -0.1in
\end{table*}

\begin{figure*}[t]
    \centering
    % \vspace{-5pt}
    \begin{subfigure}[t]{0.33\textwidth}
        \centering
        \includegraphics[width=58mm]{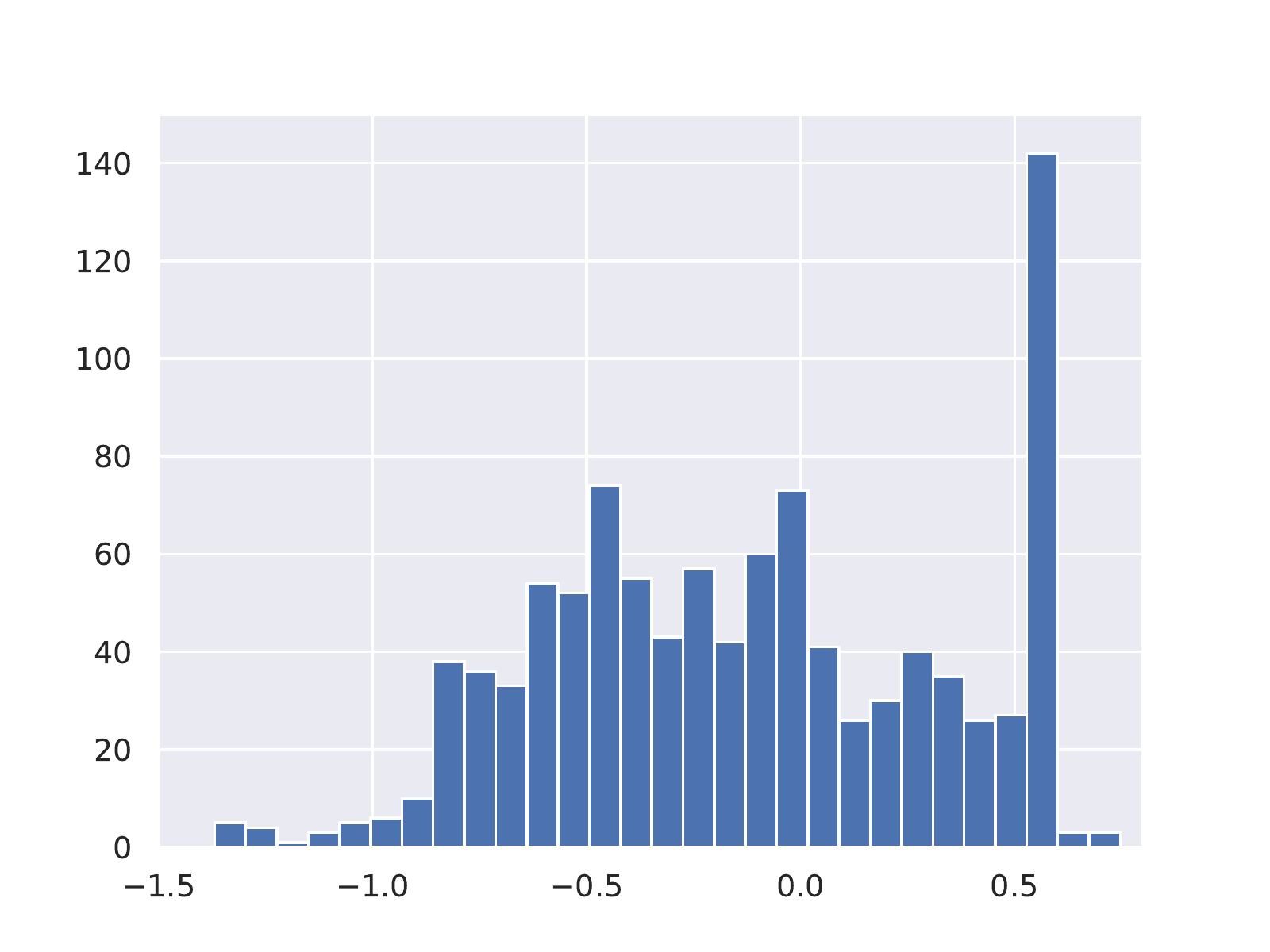}
         \captionsetup{justification=centering}
        \caption{Original Model, \\Train Aug (Source)}
        \label{subfig:orig_train}
    \end{subfigure}
    \begin{subfigure}[t]{0.33\textwidth}
        \centering
        \includegraphics[width=58mm]{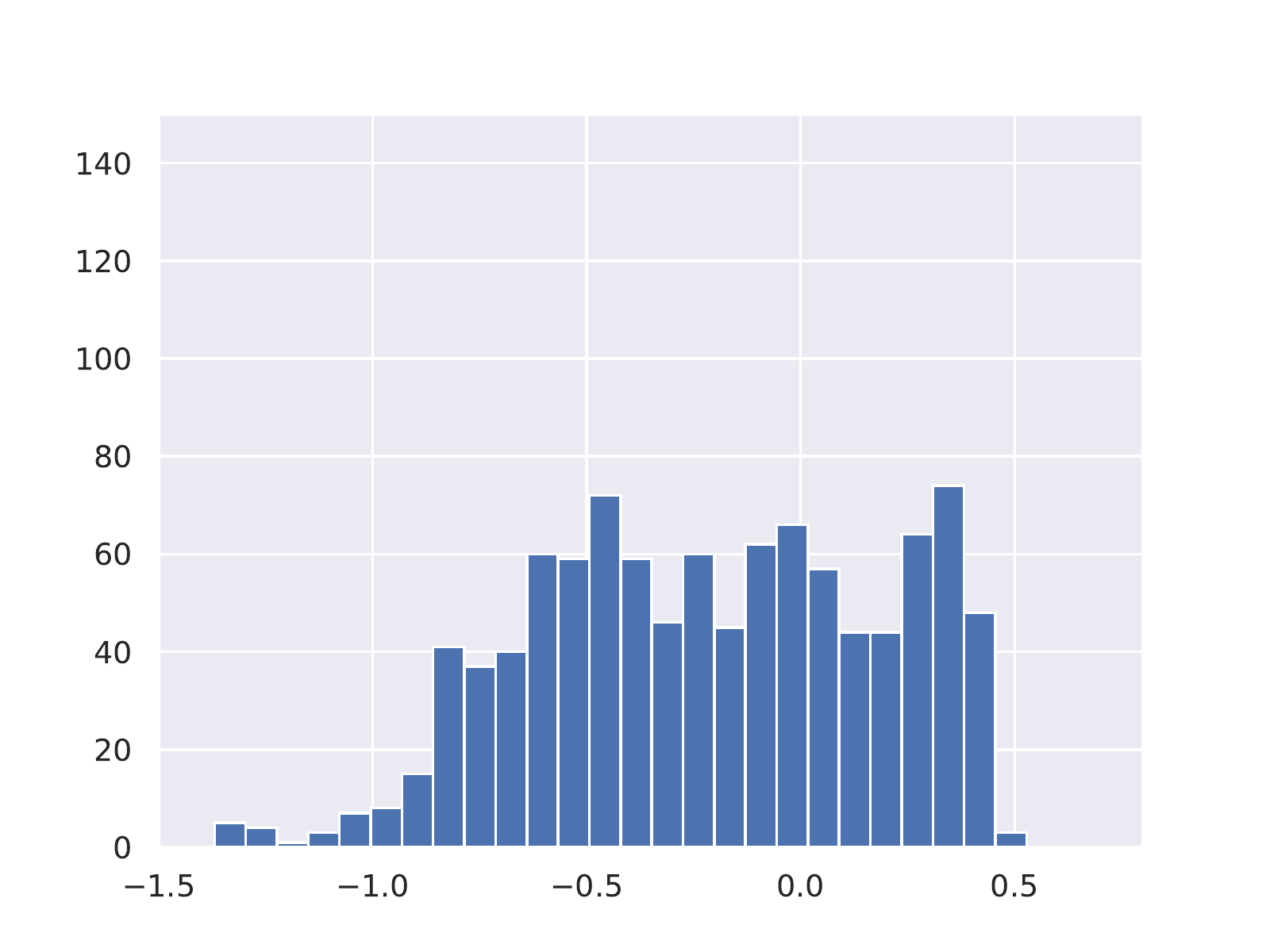}
        % \caption{Target distribution.}
         \captionsetup{justification=centering}
        \caption{Original Model, \\Test Aug (Target)}
        \label{subfig:orig_test}
    \end{subfigure}
    \begin{subfigure}[t]{0.33\textwidth}
        \centering
        \includegraphics[width=58mm]{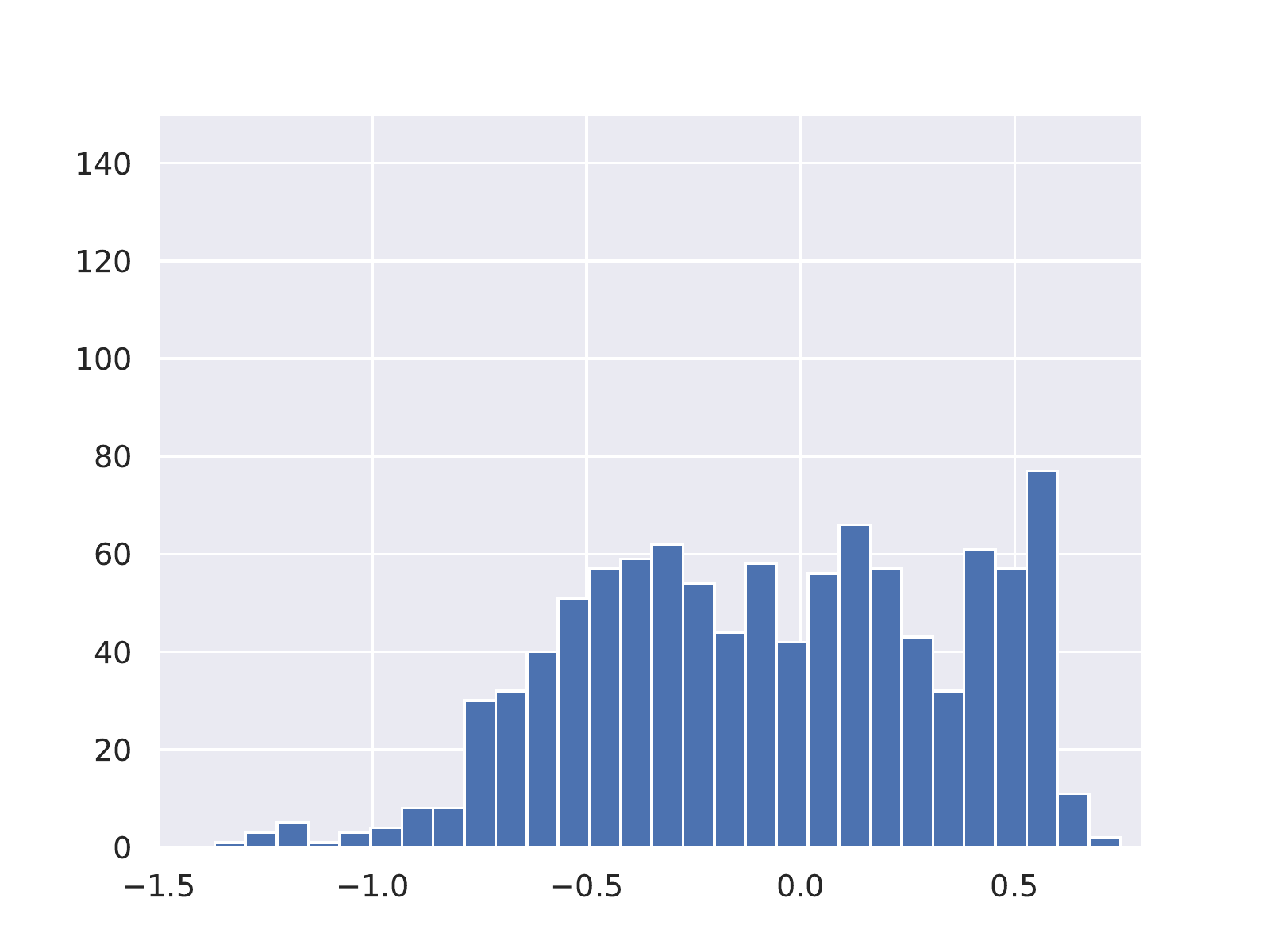}
        % \caption{Mean-aligned target distribution.}
         \captionsetup{justification=centering}
        \caption{Updated Model, \\Test Aug (Target)}
        \label{subfig:updated_test}
    \end{subfigure}
    \caption{Histograms of all activations for a single CIFAR-10 example and channel in the first Batch Normalization layer. This shows that features can be multimodal and that changes in the frequency of one mode (the spike of activations centered around $0.5$ on the source distribution in \Cref{subfig:orig_train}, which is heavily downweighted on the target distibution in \Cref{subfig:orig_test}) can cause AdaBN to shift another mode (the bulk of the activations, which shift to the right in \Cref{subfig:updated_test} relative to \Cref{subfig:orig_test}) when it would be better not to change the other activations at all.}
    \label{fig:multimodal_failure_practice}
    % \vspace{-10pt}
\end{figure*}

\paragraph{In Practice.} 
In the conceptual example described above, individual coordinates for a single class were distributed according to a mixture distribution. To exhibit an analogous shift on real data, we first identify a surprising phenomenon. We find that applying AdaBN to the original test data (or even ImageNetV2) can degrade accuracy by a few percentage points (see \Cref{tbl:adabn_accuracies}).  
This is because models use data augmentation during training but not at test time, which can lead to a discrepancy if one naively aligns the train and test sets while only applying augmentation to the former.

We find that one can prevent this decrease in accuracy by updating the Batch Normalization statistics on the target data while using the augmentation used during training time, which mimics how the original Batch Normalization statistics were computed. We denote this approach by ``AdaBN + Aug.'' We only use augmentations for aligning the features but do not use them at test time. \Cref{tbl:adabn_aug_accuracies} displays the accuracies for AdaBN + Aug on each dataset. This modification consistently improves in-distribution accuracy and often improves OOD accuracy. For example, it improves accuracy on ImageNetV2 from $60.9\%$ to $63.8\%$, compared to $63.2\%$ for standard AdaBN, showing that using training augmentations with AdaBN can be necessary for improving OOD accuracy.

What does this have to do with mixture distributions? For CIFAR-10 and TinyImageNet, standard random cropping is part of the training augmentation. See \Cref{fig:aug_example} for an example of this augmentation. Therefore, at training time an image is visibly cropped with probability $p$ and not cropped (or only slightly cropped) with probability $1-p$, leading to a mixture distribution. At test time $p$ becomes $0$. Furthermore, analogously to the conceptual example, test accuracy is high without feature alignment but becomes lower with alignment.

To better understand this distribution shift, we visualize the activations of a convolutional filter in the first Batch Normalization layer in \Cref{fig:multimodal_failure_practice}. These activations are for a network with the image in \Cref{fig:aug_example} as its input.
We find that cropping with padding causes a peak in the activations of this filter (\Cref{subfig:orig_train}) that disappears when we remove the augmentation (\Cref{subfig:orig_test}). 
Consequently, if one applies AdaBN to the test set, the mean activation for this filter on the test set is less than during training. In trying to correct for this, AdaBN increases all activations (\Cref{subfig:updated_test}). This type of shift in the distribution of activations, which we also find for other images and layers, appears responsible for the performance degradation when AdaBN is applied to the clean test set.
In short, while reweighting the training examples doesn't hurt the model, AdaBN detects a change and renormalizes the activations when it shouldn't.

\section{Understanding AdaBN \& Discussion}\label{sec:discussion}

We now explain when and why AdaBN can improve or degrade robustness. The best case scenario for AdaBN is when a shift transforms the activations in a convolutional channel, $\vec{x}_i \in \mathbb{R}^{d\times d}$, according to
\begin{equation}\label{eqn:adabn_full}
    \hat{\vec{x}}_i = a\vec{x}_i + b\cdot \vec{1}\vec{1}^T \, ,
\end{equation}
for each example $i$. Here, $d$ is the spatial height and width, and $\vec{1}$ is the $d$-dimensional vector of all ones.

This is an affine function with coefficients shared across examples and spatial dimensions.
Since this transformation corresponds to changing the mean and variance (shared across examples and spatial dimensions), AdaBN exactly inverts distribution shifts that affect the activations in this way. This characterization of AdaBN provides insight into both why it can degrade robustness for some shifts and why it improves robustness for others.

\subsection{Why Can Aligning BN Statistics Hurt?}%\label{subsec:understanding_adabn}
\Cref{eqn:adabn_full} makes it clear why the failure modes we presented can occur. Because the coefficients $a$ and $b$ are fixed across examples and spatial dimensions, AdaBN relies on the implicit assumption that different samples and spatial locations are shifted in similar ways. While this is useful because it makes it possible for AdaBN to efficiently estimate the new mean and variance under the shift, it also means that AdaBN can degrade performance when this assumption is violated. Indeed, every failure mode in \Cref{sec:failure} violates this assumption. Concerningly, it may not be easy in practice to assess whether this assumption is being violated or not. This may make it difficult to trust methods like AdaBN in high-stakes applications involving arbitrary unforeseen distribution shifts for which we need high reliability. 

\subsection{Why Can Aligning BN Statistics Help?}%\label{subsec:understanding_adabn}
In practice, methods like AdaBN that align first and second order activation statistics can yield state-of-the-art robustness \citep{Schneider2020ImprovingRA, Nado2020EvaluatingPB}.
Moreover, by \Cref{eqn:adabn_full}, one can think of AdaBN as trying to invert a change in the scale and mean of activations.
The empirical success of AdaBN, together with this interpretation of \Cref{eqn:adabn_full}, suggests that one of the main effects of some distribution shifts is to simply change the scale and mean of the network's activations in each hidden layer.
% Why might this be?

We now provide intuition for why this might be true. Consider a convolutional filter in any layer of a CNN. 
One can think of it as a feature detector that activates most for a certain input pattern. Suppose this pattern becomes more common for all inputs under the distribution shift. 
Then even if this pattern was strongly correlated with a specific class on the source distribution, it will only be weak evidence of that class on the target distribution.
This suggests we should align the means; if every image becomes greener than it was before, and green is correlated with being a frog, then under the shift we should now consider ``very green'' to be evidence of a frog but ``somewhat green'' to be uninformative. %Normalizing the mean helps do this. 

% \paragraph{Aligning the Variances.} 
Normalizing the variance, on the other hand, intuitively corrects for simple changes in the scale of the activations of a convolutional filter. For instance, if a feature becomes obscured under shift, such as if edges become blurrier, a convolutional filter that was trained to detect that feature may output activations that are closer to zero, decreasing the variance of this activation. 
However, the next layer still expects its inputs to be in a certain range. In this situation, normalizing the variance may amplify the signal that does exist by increasing the scale of the activations. 

In short, because changes in the prevalence of a pattern may result in simple changes in the activations of a feature detector for that pattern, normalization may be an effective way to partially undo the effects of some distribution shifts. This is the property that methods like AdaBN exploit.

\subsection{When Does AdaBN Help the Most?}
\citet{stylized_imagenet} argue that most ImageNet classifiers are overly reliant on the texture and style of images. 
This finding implies that most hidden layers in modern ImageNet classifiers capture low-level features such as texture or style more than they capture high-level features such as shape. Furthermore, AdaBN improves robustness by tweaking the activations of a trained network. This suggests that it is mostly ``fixing'' changes in style, since those are what activations mainly capture in the first place. Indeed, more abstract shifts in the distribution might not even register in the activations of the model because it was not trained to detect those sorts of features. 

Relatedly, \citet{li2017demystifying} draw a connection between style transfer methods and domain adaptation. They show that simply aligning the Batch Normalization statistics between two images can be used as an effective method for style transfer. Similarly, AdaBN aligns the Batch Normalization statistics between two distributions. This suggests that we can also interpret AdaBN as doing style transfer between two distributions, mapping the style of the shifted target distribution back to that of the original source distribution.

These perspectives predict that AdaBN should improve accuracy the most on distribution shifts involving changes in style and local image statistics, at least for current models, but should not substantially change performance on distribution shifts that involve more high-level, abstract changes. 
These predictions are supported by the observation that AdaBN improves accuracy much more on ImageNet-C and Stylized ImageNet than ImageNetV2. 
\Cref{tbl:adabn_accuracies} shows that AdaBN yields a relative accuracy improvement of $23\%$ for ImageNet-C and $43\%$ for Stylized ImageNet, shifts that almost exclusively involve changes in the style or texture of images, whereas it slightly degrades performance on ImageNetV2, a recollected version of ImageNet that should not have major differences in local image statistics.
This observation is further supported by the results in \citet{Schneider2020ImprovingRA}, which show that aligning BN statistics also does not help much with ImageNet-A \citep{imagenet-a} or ObjectNet \citep{barbu2019objectnet}, two other distribution shift benchmarks that, like ImageNetV2, do not primarily involve changes in local image statistics. 

These findings provide evidence for the idea that AdaBN improves robustness because it performs a sort of neural style transfer between the source and target distributions. While this makes AdaBN particularly well suited for some types of shift, such as ImageNet-C and Stylized ImageNet, it also suggests that the lackluster performance of the method on other types of distribution shifts is an inherent limitation rather than one that can be easily fixed.

\section{Conclusion}

\paragraph{Unforeseen Distribution Shifts.} Making systems robust under distribution shift is important for a wide range of applications \citep{hendrycks2019benchmarking}. UDA is considered a promising approach to this problem, but our results show that it must be used with care. For applications like self-driving cars, UDA methods should work even when applied to general, unforeseen distribution shifts. However, we find that aligning batch normalization statistics may actually degrade robustness on shifts that can arise in practice.
These limitations call into question the practical utility of aligning batch normalization statistics to improve robustness, especially for use in high-stakes applications.

\paragraph{Learning Representations.}
We also find that AdaBN disproportionately improves robustness on distribution shifts that mainly involve changes in local image statistics, such as changes in style or texture. It cannot help as much on distribution shifts involving changes in higher-level features because it only tweaks the activations of a trained network, which may not capture information about the high-level features that changed. This limitation suggests that to improve robustness for more general distribution shifts, it may be necessary to focus on learning robust representations rather than on modifying the activations of trained networks.

On the other hand, UDA methods that require additional training typically do so at test time. This is too slow for applications such as autonomous vehicles for which it is necessary to make predictions efficiently.
These drawbacks may make typical UDA methods a less promising approach to improving model robustness than other techniques that train models to have broadly robust feature representations, such as architectural changes \citep{Pan2018TwoAO} or data augmentation \citep{augmix}.

\paragraph{Future Work.}
To the best of our knowledge, there has been limited work on investigating how distribution shifts affect low-level network activations in a fine-grained way. 
Building on our work by analyzing these effects in more detail may yield additional insights into distribution shifts and the learned feature representations, and may help us develop better methods for improving robustness.

Furthermore, while there are theoretical justifications of feature alignment, they do not adequately explain when or why these methods work well in practice (\Cref{sec:related_work}). We conceptually and empirically addressed this in detail in the case of AdaBN, a particularly simple but effective method. 
Future work should more carefully identify when and why other methods for robustness are effective in practice.

Finally, we identified numerous drawbacks of current approaches to UDA. Future work should address these shortcomings to make these methods more useful and reliable for important applications.

\newpage

\paragraph{Acknowledgements.} We thank Dan Hendrycks and the anonymous reviewers for valuable feedback on earlier versions of this paper.

{\small
\setlength{\bibsep}{0pt}
\bibliographystyle{abbrvnat}
\bibliography{egbib}
}

\newpage
\appendix

\section{Additional Experimental Results}

\subsection{Uncertainty} 

Recent work has found that models become increasingly less calibrated under distribution shift \cite{uncertainty_shift}. Motivated by this problem, we also test whether AdaBN helps with calibration error on the target distribution. We use a simple and popular measure of calibration: the Expected Calibration Error (ECE) \cite{guo2017calibration}. 
In \Cref{tbl:adabn_cerr} we show the ECE for AdaBN on each dataset. We find that AdaBN substantially reduces calibration error on the corruption benchmarks and Stylized ImageNet, even cutting it in half in most cases. 

\begin{table}[th]
\caption{Expected Calibration Error (ECE) of AdaBN and variants on each shifted dataset. AdaBN substantially reduces the ECE on the corruption datasets \cite{hendrycks2019benchmarking} and Stylized ImageNet.}
\label{tbl:adabn_cerr}
% \vspace{-5pt}
\begin{center}
\begin{small}
\begin{sc}
\begin{tabular}{lcccccr}
\toprule
Method & C-10-C &  TIN-C & IN-C & INV2 & SIN  \\
\midrule
Original             & 21.5  & 25.0 & 12.0 & 10.6 & 30.9\\
AdaBN                      & 11.3 & 15.2 & 5.2 & 10.3 & 12.9\\
AdaBN + Aug             & 11.7 & 16.9 & 6.0& 10.9 & 14.4 \\
\bottomrule
\end{tabular}
\end{sc}
\end{small}
\end{center}
% \vspace{-5pt}
% \end{table*}
\end{table}

\subsection{AdaBN on subsets of classes}\label{appx:subsec:subsets}
We now provide additional results showing that applying AdaBN to subsets of classes can hurt accuracy, but that this is mitigated when one does not update the Batch Norm statistics in some of the final layers. In \Cref{fig:additional_label_shift} we show the same experiment as in Section 4.1, but this time for TinyImageNet and TinyImageNet-C. The results are qualitatively similar to those for CIFAR-10-C, though the difference between excluding the first layers vs the last layers is less dramatic for TinyImageNet-C.

\subsection{The importance of batch information}
A natural question is whether one can adapt feature alignment methods like AdaBN to a more restricted robustness setting where we do not have access to more than a single example at test time. A simple approach is to use normalization methods other than Batch Norm to align the feature distributions, but which do not use batch information. Two such methods are Group Norm \cite{group_norm} and Instance Norm \cite{instance_norm}.  Group Norm \cite{group_norm} normalizes over spatial locations and groups of multiple channels within a given layer. Instance Normalization (IN) \cite{instance_norm} was introduced for faster stylization. It normalizes over spatial locations over each channel separately, but unlike Group Norm and Batch Norm does not typically include learned affine parameters. We compared models trained using these different normalization schemes on CIFAR-10-C and TinyImageNet-C, along with the corresponding uncorrupted validation sets, and show the results in \Cref{tbl:norm_corr_acc}. In each case, we use the same architecture and hyperparameters as before, with the only difference being which normalization layer is used. For Group Norm, we test different numbers of groups ranging from $1$ to $16$, and for Instance Norm we test both with and without learned affine parameters. 

We find that the default robustness of the Batch Norm model was much lower on CIFAR-10-C than the default Instance Norm and Group norm models. However, after applying AdaBN to the Batch Norm model, its robustness ended up being higher than the other normalization methods, especially with the augmented version of AdaBN. The results for TinyImageNet are more difficult to interpret because the validation accuracy for Group Norm and especially Instance Norm are worse than for Batch Norm. Still, these results suggest that batch information can be important for improving robustness.

% \begin{table*}[th]
% \begin{table*}[th]
\begin{table}[th]
\caption{Comparing normalization methods on standard robustness benchmarks. Group Norm and Instance Norm both do worse than Batch Norm under distribution shift, even when the standard test accuracy is comparable.}
\label{tbl:norm_corr_acc}
\vskip 0.15in
\begin{center}
\begin{small}
\begin{sc}
\begin{tabular}{lccccr}
\toprule
Method & C-10 & C-10-C & TIN & TIN-C\\
\midrule
Original model        & 94.82 & 72.31 & 63.80 & 24.77\\
AdaBN                 & 92.84 & 83.63 & 60.32 & 40.11\\
AdaBN + Aug        & 94.84 & 86.78  & 64.05 & 41.80\\
IN (no affine) & 92.68 & 81.52  & 29.54 & 11.04\\
IN (affine) & 93.51 & 81.43  & 45.32 & 17.04 \\
GN ($1$ group) & 92.53 & 76.76  & 56.45 & 22.14 \\
GN ($4$ groups) & 93.32 & 78.15  & 59.34 & 23.18\\
GN ($16$ groups) & 93.85 & 81.68  & 58.11 & 22.91\\
\bottomrule
\end{tabular}
\end{sc}
\end{small}
\end{center}
\vskip -0.1in
\end{table}
% \end{table*}

\section{Further Discussion}
\paragraph{Covariate Shift}
Researchers have attempted to identify assumptions that are sufficient for successful unsupervised domain adaptation. One assumption that has been considered is covariate shift, i.e. $p_S(y | x) = p_T(y | x)$. \citet{ben2010impossibility} showed that covariate shift is not sufficient for UDA, even when paired with either (i) the assumption that $p_S(x) \approx p_T(x)$ or (ii) the assumption that there is a classifier in the hypothesis class with low error on both domains. 

The failures we present can occur even under the covariate shift assumption and even assuming there is no label shift (also known as prior shift or target shift). For shifted spatial locations, this is immediately true because we just made $x_2 = 0$, when $x_2$ didn't depend on the label in the first place. These two assumptions can also hold for the shifted examples failure; in the simplest case, this is is true when $p(y = -1) = 1$. 

The covariate shift assumption is less clear with our failure modes on real data. Nevertheless, it should at least approximately hold in these cases, and can be modified to exactly hold. In particular, while both real shifts (black border and data augmentation) can cut out some relevant features, they rarely change the ground truth label. 

\section{Theoretical results}
\subsection{Target error bounds can be uninformative}\label{appx:subsec:uninformative}
Denote the target and source classification errors by $\epsilon_T(h)$ and $\epsilon_S(h)$ respectively, and denote the optimal joint error by $\lambda := \min_{h \in \mathcal{H}} \epsilon_T(h) + \epsilon_S(h)$. \citet{ben2010theory} show that for any $h \in \mathcal{H}$,
\begin{equation}\label{eqn:ben_david_thm}
    \epsilon_T(h) \leq \epsilon_S(h) + \lambda + |\epsilon_T(h, h^*) - \epsilon_S(h, h^*)| \,,
\end{equation}
where $\epsilon_S(h, h^*) = Pr_{x \sim D_S}[h(x) \neq h^*(x)]$ and $\epsilon_T(h, h^*) = Pr_{x \sim D_T}[h(x) \neq h^*(x)]$. \citet{ben2010theory} also upper bound $|\epsilon_T(h, h^*) - \epsilon_S(h, h^*)|$ in terms of a distance $d_{\mathcal{H}\Delta\mathcal{H}}(D_S, D_T)$ between $D_S$ and $D_T$,
\begin{equation}
    d_{\mathcal{H}\Delta\mathcal{H}}(D_S, D_T) = \sup_{h \in \mathcal{H}} |\epsilon_T(h, h^*) - \epsilon_S(h, h^*)| \,.
\end{equation}

Many methods aim to minimize $\epsilon_S(h)$ and $d_{\mathcal{H}\Delta\mathcal{H}}(D_S, D_T)$. In practice $\lambda$ is an unknown quantity that depends on the true target labeling function, so most feature alignment methods ignore it. However, this makes it unclear whether this bound provides much of a guarantee even for methods that were directly inspired by it. 

We now show that even if one \emph{does} make $\lambda$ small, such as by using a flexible class $\mathcal{H}$ of neural networks, then the bound proved by \citet{ben2010theory} can be uninformative for a different reason. In particular, when $\lambda = 0$ the bound is equivalent to the triangle inequality. Specifically, when $\lambda = 0$, this means that $\epsilon_T(h, h^*) = \epsilon_T(h)$ and $\epsilon_S(h, h^*) = \epsilon_S(h)$. Hence, the bound in \Cref{eqn:ben_david_thm} reduces to
\begin{equation}
    \epsilon_T(h) \leq \epsilon_S(h) + |\epsilon_T(h) - \epsilon_S(h)| \,,
\end{equation}
which is always true. Upper bounding this in terms of $d_{\mathcal{H}\Delta\mathcal{H}}(D_S, D_T)$ is then equivalent to:
\begin{equation}
    \epsilon_T(h) \leq \epsilon_S(h) + \sup_{h \in \mathcal{H}}|\epsilon_T(h) - \epsilon_S(h)| \,,
\end{equation}
which is still uninformative. 

Other generalization bounds have been proven, such as by \citet{on_invariant_reps, johansson2019support}, but these also don't explain why aligning the feature distributions helps in practice. \citet{on_invariant_reps} essentially replace $\lambda$ with a term that captures the difference between the true source and target labeling functions. \citet{johansson2019support} prove a bound based on the support of the source and target distributions that explicitly accounts for the non-invertibility of the feature representation. However, both bounds still include an unobservable quantity that feature alignment methods ignore. Neither paper explains why these unobservable terms should be small in practice for such methods.

\subsection{AdaBN for approximately affine shifts}\label{appx:subsec:approx_affine}
We saw that AdaBN exactly removes shifts that are characterized by a particular type of affine transformation. We now bound how well it removes a shift that is only approximately characterized by such a transformation. For simplicity we focus on the one dimensional setting.

\begin{theorem}\label{thm:thm1}
Suppose $D_S$ is some source distribution. Define $D_T$ by sampling $x \sim D_S$ and letting $\tilde{x} = ax + b + \epsilon$, where $a > 0$ and $b$ are constant, and $\epsilon$ is an arbitrary zero-mean random variable. Let $\hat{\mu}$ and  $\hat{\vec{\sigma}}^2$ be the mean and variance of $\tilde{x}$. Define $\hat{x} = (\tilde{x} - \hat{\mu}) \cdot ( \sigma / \hat{\sigma}) + \mu$, where $\mu$ and $\sigma$ are the mean and standard deviation of $x$. Assume $|\epsilon| \leq r$ and define $\delta = \frac{1}{a\sigma}\sqrt{2ar|x| + r^2}$. If $\delta < 1$, then:
\begin{equation}\label{eqn:thm1}
    |\hat{x} - x| \leq 2|x|\delta + \frac{r}{a}(1 + 2\delta)
\end{equation}
Moreover, if we additionally assume that $\expec[\epsilon x] = 0$ ($\epsilon$ and $x$ are uncorrelated), then \Cref{eqn:thm1} holds for $\delta = \frac{r}{a\sigma}$.
\end{theorem}

\begin{proof}
Assume without loss of generality that $\mu = 0$. Then we can easily estimate and subtract $\expec[\tilde{x}] = b$, so we may also assume that $b=0$. Hence, $\tilde{x} = ax + \epsilon$ has mean zero, so
\begin{equation}
    \hat{\sigma}^2 = \expec[\tilde{x}^2] = a^2\sigma^2 + 2a\expec[x\epsilon] + \expec[\epsilon^2]
\end{equation}
By assumption, $|\epsilon| \leq r$, so $|\expec[x\epsilon]| \leq r|x|$. Hence:
\begin{equation}\label{eqn:sigma_bound}
    a^2\sigma^2 - 2ar|x| \leq \hat{\sigma}^2 \leq a^2\sigma^2 + 2ar|x| + r^2
\end{equation}

This implies:
\begin{equation}
    |\hat{x} - x| = |(ax + \epsilon) \cdot \frac{\sigma}{\hat{\sigma}} - x| = |x(a \frac{\sigma}{\hat{\sigma}} - 1) + \epsilon \frac{\sigma}{\hat{\sigma}}|
\end{equation}
\begin{equation}\label{eqn:diff_bound}
    \leq |x||a\sigma / \hat{\sigma} - 1| + r|\sigma / \hat{\sigma}|
\end{equation}

But by \cref{eqn:sigma_bound}, we have:
\begin{equation}
    \hat{\sigma} / a\sigma \in \sqrt{1 \pm \frac{1}{a^2\sigma^2}(2ar|x| + r^2)}
\end{equation}

Let $\delta := \frac{1}{a^2\sigma^2}(2ar|x| + r^2)$. By assumption, $\delta \in [0, 1)$. Hence, $1-\sqrt{\delta} \leq \sqrt{1-\delta}$ and $\sqrt{1+\delta} \leq 1+\sqrt{\delta}$. This implies
\begin{equation}
    \hat{\sigma} / a\sigma \in 1 \pm \sqrt{\delta}
\end{equation}

Finally, using that for $z \in [0, 1]$, $\frac{1}{1+z} \geq 1-z$ and $\frac{1}{1-z} \leq 1+2z$, this implies
\begin{equation}
     a\sigma / \hat{\sigma}  \in 1 \pm 2\sqrt{\delta}
\end{equation}

Combining this with \cref{eqn:diff_bound} yields
\begin{equation}
    |\hat{x} - x| \leq 2|x|\sqrt{\delta} + \frac{r}{a}(1 + 2\sqrt{\delta})
\end{equation}
Replacing $\sqrt{\delta}$ with $\delta$ yields the desired result. When $x$ and $\epsilon$ are uncorrelated, the $2arx$ term becomes zero and everything else remains the same, so this term just disappears from $\delta$. 
\end{proof}

The theorem bounds the difference between a shifted but normalized input $\hat{x}$ and the original unshifted input $x$. The bound suggests two things. First, $\frac{r}{a}$ should be small. In other words, the scale of the error, $r$, should not be too large relative to the magnitude of rescaling, $a$, since otherwise the error terms will dominate the ``signal'' after rescaling. Moreover, it precisely describes in what sense the errors $\epsilon$ should be well-behaved. The bound becomes tighter when both $\epsilon$ and $x$ are decorrelated, and when the scale of $\epsilon$ is small.

\begin{figure*}[ht]
    \centering
    \begin{subfigure}[t]{0.49\linewidth}
        \centering
        \includegraphics[scale=0.40]{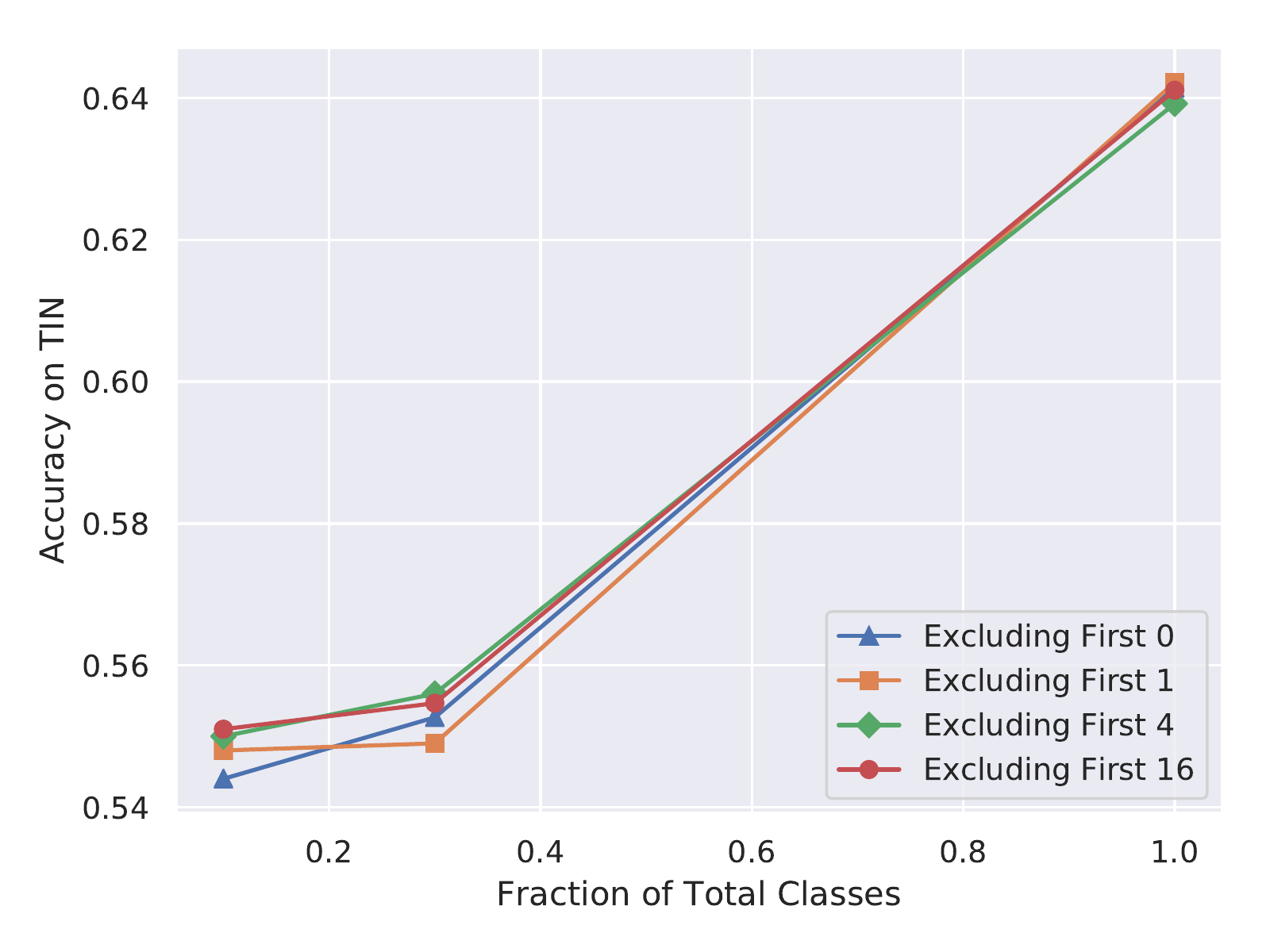}
        \caption{TinyImageNet Excluding First Layers}
    \end{subfigure}
    \begin{subfigure}[t]{0.49\linewidth}
        \centering
        \includegraphics[scale=0.40]{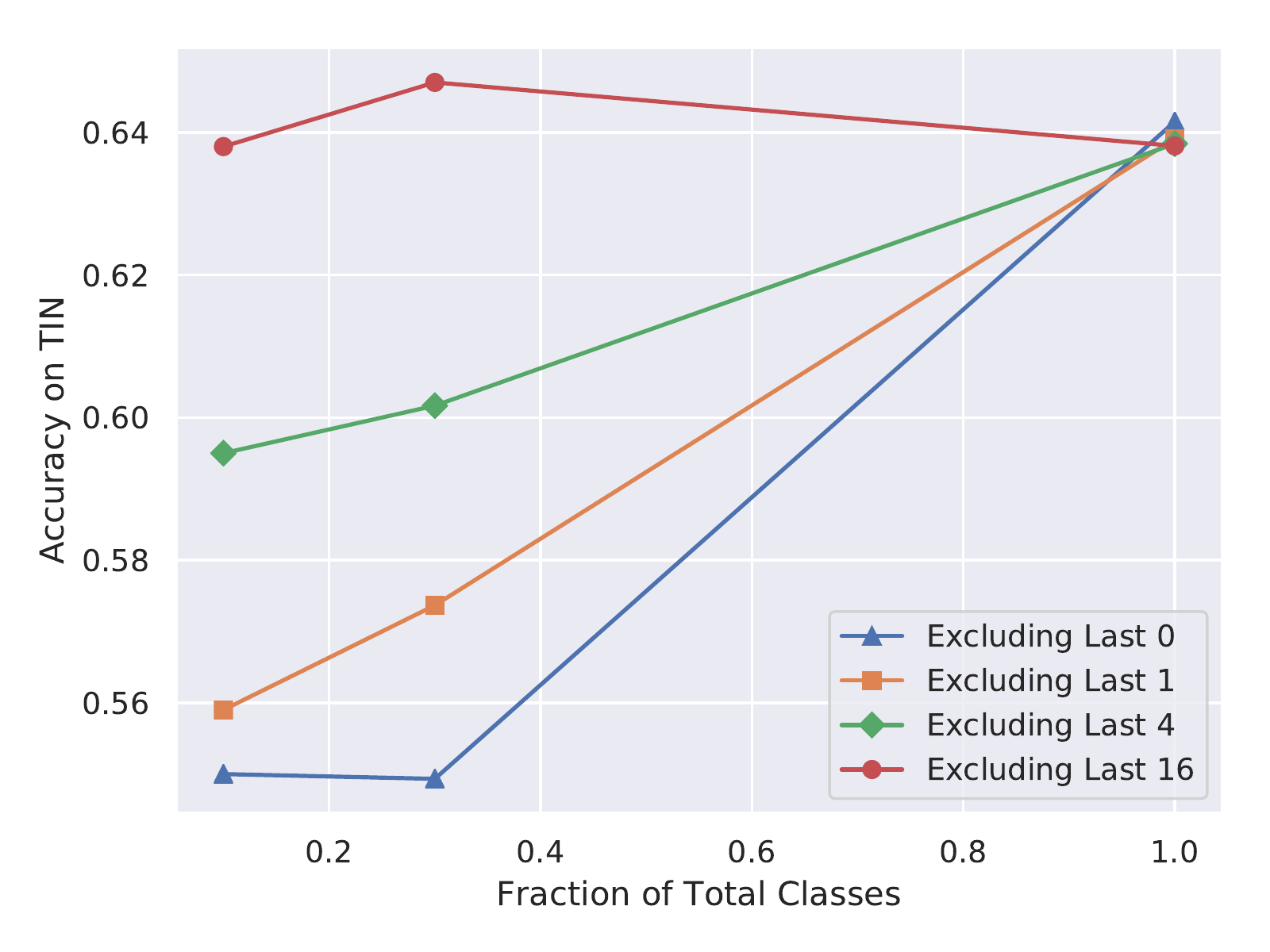}
        \caption{TinyImageNet Excluding Last Layers}
    \end{subfigure}
    
    \begin{subfigure}[t]{0.49\linewidth}
        \centering
        \includegraphics[scale=0.40]{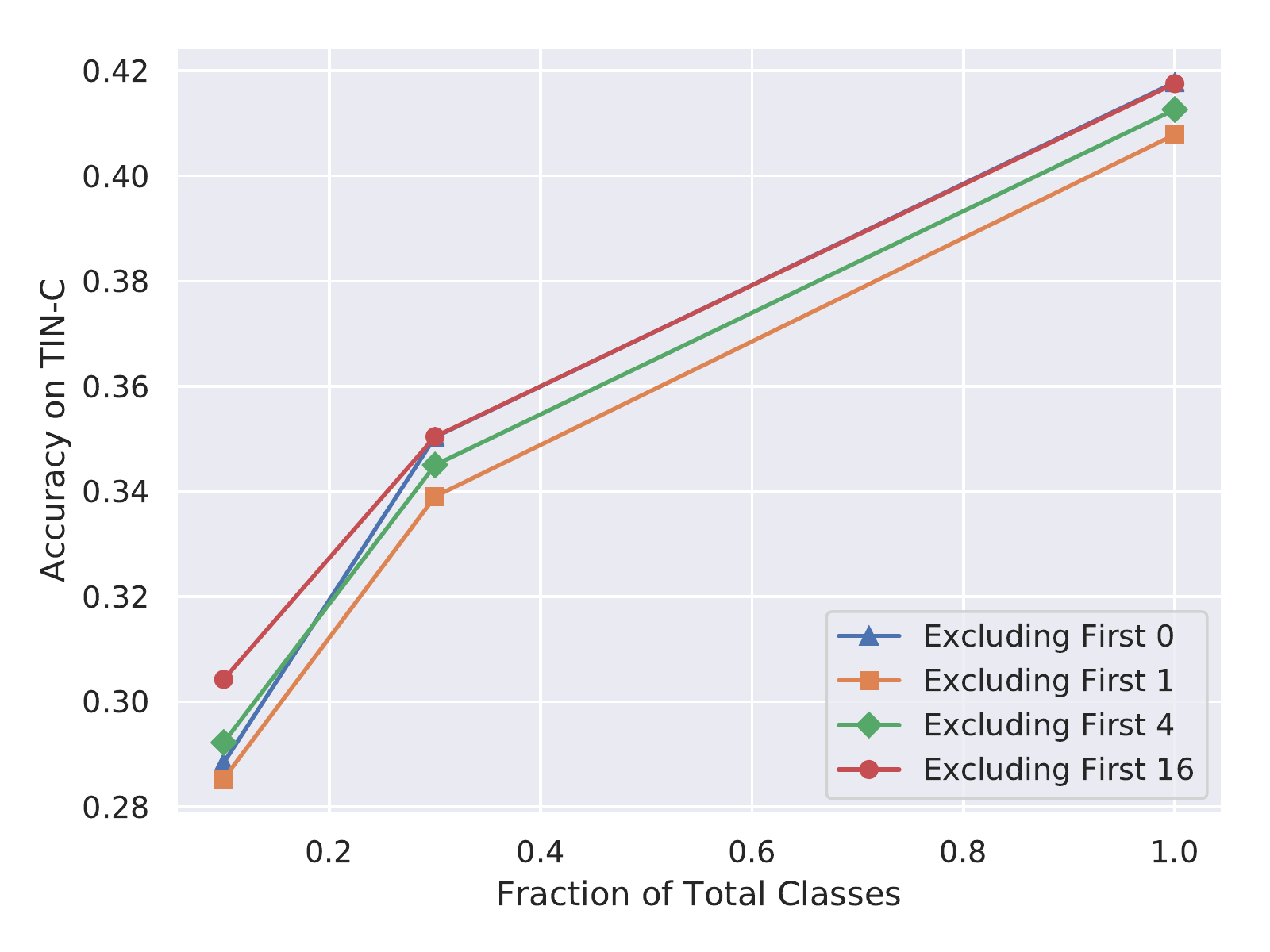}
        \caption{TinyImageNet-C Excluding First Layers}
    \end{subfigure}
    \begin{subfigure}[t]{0.49\linewidth}
        \centering
        \includegraphics[scale=0.40]{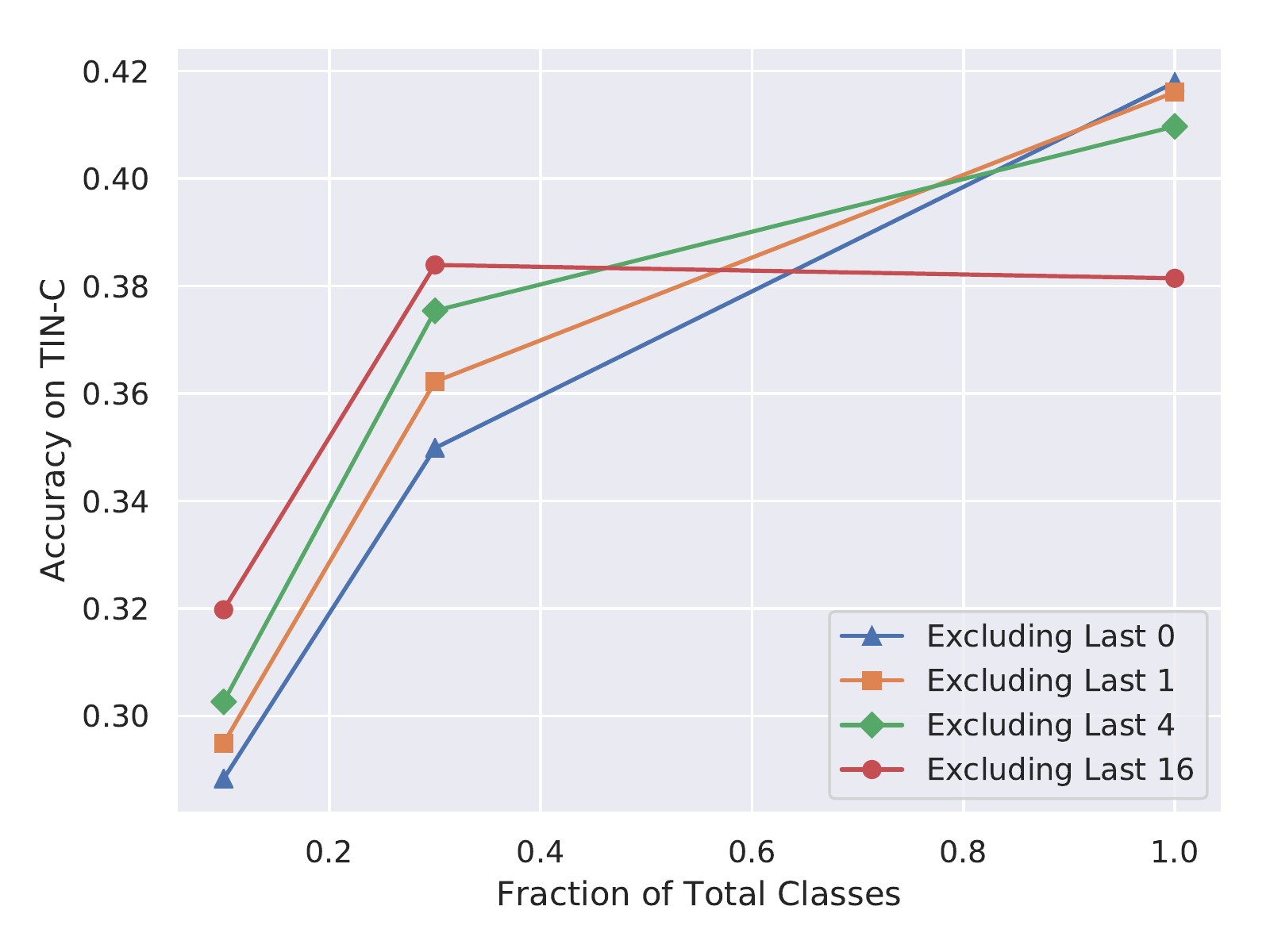}
        \caption{TinyImageNet-C Excluding Last Layers}
    \end{subfigure}
    \caption{The effect of updating the Batch Norm statistics using AdaBN + Aug on different subsets of classes for TinyImageNet and TinyImageNet-C. The results are qualitatively similar to those found in \Cref{fig:label_shift}}
    \label{fig:additional_label_shift}
\end{figure*}

\end{document}